\documentclass[runningheads]{llncs}
\usepackage{graphicx}
\usepackage[T1]{fontenc}
\usepackage{amsfonts}
\usepackage{amsmath}
\usepackage[inline]{enumitem}
\usepackage{xspace}

\usepackage{times}

\usepackage{mathnotation}
\usepackage{ensuretext}
\usepackage{booktabs}
\usepackage{multirow}
\usepackage{graphicx}
\usepackage{colortbl}
\usepackage{subfig}

\usepackage{rotating}
\usepackage{blkarray}

\usepackage{threeparttable}

\usepackage{algorithm}
\usepackage[noend]{algpseudocode}
\algrenewcommand\algorithmicrequire{\textbf{Input:}}
\algrenewcommand\algorithmicensure{\textbf{Output:}}
\algrenewcommand\alglinenumber[1]{\tiny #1:} 
\algnewcommand{\IfThen}[2]{
	\State \algorithmicif\ #1\ \algorithmicthen\ #2
}

\newcommand{\dpi}{\mathit{fpi}}

\newcommand{\dt}{\mathcal{D}^*}
\newcommand{\md}{\mathcal{D}}
\newcommand{\mD}{\mathbf{D}}

\newcommand{\Pt}{\mathfrak{P}}
\newcommand{\oracle}{\mathsf{expert}}

\newcommand{\mt}{\mathcal{T}\xspace}
\newcommand{\ma}{\mathcal{A}\xspace}

\newcommand{\tax}{\mathit{ax}}
\newcommand{\mo}{\mathcal{O}}

\newcommand{\mb}{\mathcal{B}}
\newcommand{\Tp}{\mathit{P}}
\newcommand{\Tn}{\mathit{N}}
\newcommand{\tp}{\mathit{p}}
\newcommand{\tn}{\mathit{n}}
\newcommand{\dx}[1]{{\bf D}_{#1}^+}
\newcommand{\dnx}[1]{{\bf D}_{#1}^{-}}
\newcommand{\dz}[1]{{\bf D}_{#1}^0}

\begin{document}
%
\title{A New Expert Questioning Approach to More Efficient Fault Localization in Ontologies\thanks{This is a preprint of the work \protect\cite{rodler2022onestep} that is formally published in the \emph{Knowledge-Based Systems} journal.}}
%
\titlerunning{Expert Questioning for Ontology Fault Localization}
%
\author{Patrick Rodler\orcidID{0000-0001-8178-3333} \and
Michael Eichholzer}
\authorrunning{Patrick Rodler and Michael Eichholzer}
%
\institute{Alpen-Adria Universit\"at Klagenfurt, 9020 Klagenfurt, Austria \\ \email{patrick.rodler@aau.at} \\
\email{michael.eichholzer@aon.at}}
\maketitle              
\begin{abstract}
When ontologies reach a certain size and complexity, faults such as inconsistencies, unsatisfiable classes or wrong entailments are hardly avoidable. 
Locating the incorrect axioms that cause these faults is a hard and time-consuming task.
Addressing this issue, several techniques for semi-automatic fault localization in ontologies have been proposed. 
Often, these approaches involve a human expert 
who provides answers to system-generated questions about the intended (correct) ontology in order to reduce the possible fault locations.
To suggest as informative questions as possible, existing methods draw on various algorithmic optimizations as well as heuristics. 
However, these computations are often based on certain 
assumptions about the interacting user. 

In this work, we characterize and discuss different user types and show that existing approaches do not achieve optimal efficiency for all of them. As a remedy, we suggest a new type of expert question which aims at fitting the answering behavior of all analyzed experts. Moreover, we present an algorithm to optimize this new query type which is fully compatible with the (tried and tested) heuristics used in the field. Experiments on faulty real-world ontologies 
show the potential of the new querying method for minimizing the expert consultation time, independent of the expert type. Besides, the gained insights can inform the design of interactive debugging tools towards better meeting their users' needs.

\end{abstract}
\section{Introduction}
\label{sec:intro}
As Semantic Web technologies have become widely adopted in, e.g., government, 
security
and health applications, the quality assurance of the data, information and knowledge used by these applications is a critical requirement. At the core of semantic web technologies, ontologies are a means to represent knowledge in a formal, structured and human-readable way, with a 
well-defined semantics. As ontologies are often developed and cured in a collaborative way 
by numerous contributors \cite{tudorache2008supporting,smith2007obo} possibly not sharing their conceptualization of the domain of interest,
are merged by automated alignment tools \cite{meilicke2011}, reach sizes and complexities exceeding human reasoning and understanding capabilities \cite{golbeck2003NCIT}, or use expressive logical formalisms such as OWL~2 \cite{grau2008owl}, 
faults occur regularly 
during the evolution of ontologies \cite{meilicke2011,schulz2010pitfalls,ceusters2005terminological,rector2011getting}. 
Since one of the major benefits of ontologies is the capability of using them to perform logical reasoning and thereby solve relevant problems, faults that affect the ontology's semantics are of particular concern for semantic applications. Specifically, such faults may cause the ontology, e.g., to become inconsistent, include unsatisfiable classes or feature wrong entailments. 

One important step towards the repair of such faults is the \emph{localization} of the responsible faulty axioms. 
To handle nowadays ontologies with often thousands of axioms, several fault localization approaches \cite{meilicke2011,Shchekotykhin2012,Kalyanpur2006a,Nikitina11} have been proposed to semi-automatically assist humans in this
complex and time-consuming task.  
These approaches, which are mainly based on the \emph{model-based diagnosis} framework \cite{Reiter87,dekleer1987}, use the faulty ontology along with additional specifications to reason about different fault assumptions. Such fault assumptions are called \emph{diagnoses} if they are consistent with all 
given
specifications. 
The specifications usually comprehend some requirements to the correct ontology, e.g., in the form of \emph{logical properties} (e.g., consistency, coherency), and/or in terms of necessary and forbidden entailments. The latter are usually referred to as \emph{positive and negative test cases} \cite{felfernig2004consistency,Shchekotykhin2012,DBLP:conf/icbo/SchekotihinRSHT18a}. 

%
%
Research on model-based diagnosis has brought up various algorithms \cite{Reiter87,dekleer1987,Kalyanpur2006a,Rodler2015phd,meilicke2011,shchekotykhin2014sequential} for computing and ranking diagnoses; however, a frequent problem is that a high number of competing diagnoses might exist 
where all of them lead to repaired ontologies with necessarily different semantics \cite{Rodler2015phd}. Finding the correct diagnosis (pinpointing the actually faulty axioms) is thus crucial for successful and sustainable repair. Since it is a mentally-demanding task for humans to recognize and reason about entailments and non-entailments \cite{horridge2011cognitive} of the ontology under particular fault assumptions, interactive techniques\footnote{Depending on the community, these techniques are referred to as Sequential Diagnosis and Interactive (or: Test-Driven) Ontology (or: Knowledge Base) Debugging.} \cite{Shchekotykhin2012,Rodler2015phd}
have been developed to undertake this task and relieve the 
user as much as possible. What remains to be accomplished by the interacting human---usually an ontology engineer or a domain expert (referred to as \emph{expert} in the sequel)---is the answering of a series of \emph{queries} about the intended ontology that are shown to them by the system. Roughly, that means the user has to classify certain axioms as either entailments (positive test cases) or non-entailments (negative test cases) of the intended ontology. A concrete implementation of such a query-based fault localization approach is \emph{OntoDebug}\footnote{All information about OntoDebug can be found at \url{http://isbi.aau.at/ontodebug/}} \cite{schekotihin2018ontodebug}, a plug-in for the popular ontology editor \emph{Prot\'eg\'e} \cite{noy2003protege}.

Several evaluations \cite{Shchekotykhin2012,shchekotykhin2014sequential,DBLP:journals/corr/rodler17jair} 
have shown the feasibility and usefulness of 
query-based 
fault localization, and its efficiency has been improved by various algorithmic optimizations \cite{jannach2016parallel,shchekotykhin2015mergexplain,rodler2018socs,rodler-dx17} and the use of heuristics \cite{Shchekotykhin2012,Rodler2013,rodler17dx_activelearning,rodler2018ruleML,DBLP:journals/corr/Rodler16a} for the selection of the most informative questions to ask an expert. 
%
However, the used heuristics, algorithms and optimization criteria are based on certain assumptions about the question answering behavior of experts. 
In this work, we critically discuss existing approaches with regard to these assumptions. Particularly, we characterize different types of experts and show that not all of them are equally well accommodated by current querying approaches. That is, we observe that the necessary expert interaction cost to locate the ontology's faults is significantly influenced by the way queries posed by the debugging system are answered. 
To overcome this issue, we propose a new way of user interaction that serves all discussed expert types equally well 
and moreover increases the expected amount of information relevant for fault localization obtained from the expert per asked axiom.
%
In addition, we present a polynomial time and space algorithm to generate and optimize the newly suggested type of question in terms of the well-understood and proven heuristics 
used in the field. 

The main idea behind the new approach is to restrict questions---which are, for quite natural reasons, \emph{sets of} axioms in existing methods---only to \emph{single} axioms, as usually done in \emph{sequential diagnosis} applications \cite{dekleer1987,siddiqi2007hierarchical}, where systems different from ontologies (e.g., digital circuits) are analyzed and such singleton queries are the natural choice. That is, experts are asked single axioms at a time instead of getting batch queries which (possibly) include multiple axioms.
Experiments on real-world faulty ontologies manifest the reasonability and usefulness of the new approach. 
Specifically, we find that, in more than two thirds of the 
studied cases,
the new querying technique is superior to existing ones in terms of minimizing the number of required expert inputs, regardless of the type of expert. In addition, the time for the determination of the best next query is reduced by at least 80\,\% in all investigated cases when using singleton queries instead of existing techniques.
%

The rest of the work is organized as follows. In Section~\ref{sec:basics}, we give a short introduction to query-based fault localization in ontologies, before we challenge certain assumptions made by state-of-the-art approaches in the field in Section~\ref{sec:discussion_of_existing_approaches}. We describe our proposed approach in Section~\ref{sec:new_approach}, where we also discuss its pros and cons, and elaborate an algorithm for the computation of the suggested new query type.
Our experiments and the obtained results are explicated in Section~\ref{sec:eval}. Finally, we point to open questions and both interesting and promising future research issues in Section~\ref{sec:research_limitations}, before we 
summarize the conclusions from this work
in Section~\ref{sec:conclusion}.

\section{Query-Based Fault Localization in Ontologies}
\label{sec:basics}
We briefly 
recap
basics of query-based 
ontology fault localization,
based on 
\cite{Rodler2015phd,Shchekotykhin2012}. As a running example we reuse the example presented in \cite{rodler2018ruleML}. 

\noindent\textbf{Fault Localization Problem Instance.} 
We assume a faulty ontology to be given by the finite set of axioms $\mo \cup \mb$, where $\mo$ includes the \emph{possibly faulty} axioms and $\mb$ the \emph{correct} (background knowledge) axioms, and $\mo \cap \mb = \emptyset$ holds. 
This partitioning of the ontology means that faulty axioms must be sought only in $\mo$, whereas $\mb$ provides the fault localization context. At this, $\mb$ can 
be useful to achieve a fault search space restriction (if parts of the faulty ontology are marked correct) or a higher fault detection rate (if external approved knowledge is taken into account, which may point at otherwise undetected faults). 
%
Besides logical properties such as consistency and coherency, requirements to the intended (correct) ontology can be formulated as a set of test cases \cite{felfernig2004consistency}, analogously as it is common practice in software engineering \cite{beck2003test}. In particular, we distinguish between two types of test cases, positive (set $\Tp$) and negative (set $\Tn$) ones. Each test case is a set (interpreted as conjunction) of axioms; positive ones $\tp\in\Tp$ \emph{must} be and negative ones $\tn\in\Tn$ \emph{must not} be entailed by the intended ontology. We call $\tuple{\mo,\mb,\Tp,\Tn}$ an \emph{(ontology) fault localization problem instance (FPI)}.
\begin{example}\label{ex:DPI}
	Consider the following ontology with the terminology $\mt$:
	\begin{center}
		\setlength{\tabcolsep}{2.8mm}
		\begin{tabular}{rlll}
			$\{$&\multicolumn{2}{l}{$\tax_1 : \mathit{ActiveResearcher} \sqsubseteq \exists writes.(\mathit{Paper}\sqcup\mathit{Review})\; ,$} & \\ 
			& $\tax_2 : \exists writes.\top \sqsubseteq \mathit{Author}\; ,$ &
			$\tax_3 : \mathit{Author} \sqsubseteq \mathit{Employee} \sqcap \mathit{Person}$ & $\}$ 
		\end{tabular}
	\end{center}
	and assertions $\ma :\{\tax_4: \mathit{ActiveResearcher}(\mathit{ann})\}$. To locate faults in the terminology while accepting as correct the assertion and stipulating that Ann is not necessarily an employee (negative test case $\tn_1 : \setof{\mathit{Employee}(\mathit{ann})}$), one can specify the following FPI: $\dpi_{ex}:=\tuple{\mt,\ma,\emptyset,\setof{\tn_1}}$.\qed 
\end{example}

\noindent\textbf{Fault Hypotheses.} 
Let $U_{\Tp} := \bigcup_{\tp\in\Tp} \tp$ and $\mathbf{C}_{\bot} := \{C\sqsubseteq\bot \mid C \text{ named class in }\mo, \mb$ $\text{or }\Tp\}$. 
Given that the ontology, along with the positive test cases, is inconsistent or incoherent, i.e., $\mo \cup \mb \cup U_{\Tp} \models x$ for some $x \in \setof{\bot}\cup \mathbf{C}_{\bot}$, or some negative test case is entailed, i.e., $\mo \cup \mb \cup U_{\Tp} \models \tn$ for some $\tn \in \Tn$, some axioms in $\mo$ must be 
accordingly modified or deleted to enable the formulation of the intended ontology. We call such a set of axioms $\md \subseteq \mo$ a \emph{diagnosis} for the FPI $\tuple{\mo,\mb,\Tp,\Tn}$ iff $(\mo \setminus \md) \cup \mb \cup U_{\Tp} \not\models x$ for all $x \in \Tn \cup \setof{\bot}\cup \mathbf{C}_{\bot}$. $\md$ is a \emph{minimal diagnosis} 
iff there is no diagnosis $\md' \subset \md$. 
We call $\dt$ \emph{the actual diagnosis} iff all $\tax \in \dt$ are faulty and all $\tax \in \mo\setminus\dt$ are correct. 
For efficiency and to suggest changes to the faulty ontology that preserve as much of its meaning as possible,
fault localization approaches usually restrict their focus to the computation of minimal diagnoses.
\begin{example}\label{ex:diagnoses}
	For $\dpi_{ex}=\tuple{\mo,\mb,\Tp,\Tn}$ from Example~\ref{ex:DPI}, 
	$\mo \cup \mb \cup U_\Tp$ entails the negative test case $\tn_1 \in\Tn$, i.e., that Ann is an employee. The reason is that according to $\tax_1 (\in \mo)$ and $\tax_4 (\in \mb)$, Ann writes some paper or review since she is an active researcher. Due to the additional $\tax_2 (\in \mo)$, Ann is also an author because she writes something. Finally, since Ann is an author, she must be both an employee and a person, as postulated by $\tax_3 (\in \mo)$. Hence, $\md_1 := [\tax_1]$, $\md_2:=[\tax_2]$, $\md_3:=[\tax_3]$ are (all the) minimal diagnoses for $\dpi_{ex}$, as the deletion of any $\tax_i \in \mo$ breaks the unwanted entailment $\tn_1$.\qed
\end{example}

\noindent\textbf{Eliminating Wrong Fault Hypotheses.}
%
The main idea model-based diagnosis systems use for fault localization, i.e., to find the actual diagnosis among the set of all (minimal) diagnoses, is that different fault assumptions have (necessarily \cite{Rodler2015phd}) different semantic properties in terms of entailments and non-entailments. This fact can be exploited to distinguish between diagnoses by asking an expert whether a (set of) axiom(s) $Q$, which is entailed by some and inconsistent with some other fault assumptions, must be correct or not. 
More formally, given a known set of minimal diagnoses $\mD$, a \emph{(normal) query} (wrt.\ $\mD$) is a set of axioms $Q$ that rules out at least one diagnosis in $\mD$, both if $Q$ is classified as a positive test case
and if $Q$ is classified as a negative test case. 
That is, at least one $\md_i \in \mD$ is not a diagnosis for $\tuple{\mo,\mb,\Tp\cup\setof{Q},\Tn}$ and at least one diagnosis $\md_j \in \mD$ is not a diagnosis for $\tuple{\mo,\mb,\Tp,\Tn\cup\setof{Q}}$. 
A query $Q$ corresponds to the question ``Is (the conjunction of axioms in) $Q$ an entailment of the intended ontology?''. 
The expert who provides answers to queries can be modeled as a function $\oracle: \mathbf{Q} \to \setof{y,n}$ where $\mathbf{Q}$ is the query space; $\oracle(Q) = y$ iff the answer to the question is positive, else $\oracle(Q) = n$.

Every set of axioms $X$ partitions any set of diagnoses $\mD$ for an FPI $\tuple{\mo,\mb,\Tp,\Tn}$ into three subsets---the diagnoses predicting that $X$ is a positive test case (set $\dx{X} \subseteq \mD$), the ones predicting that $X$ is a negative test case (set $\dnx{X} \subseteq \mD$), and the ones that do not predict any classification for $X$ (set $\dz{X} \subseteq \mD$). More specifically, among the diagnoses in $\mD$, $\dx{X}$ comprises exactly the diagnoses that are no diagnoses for $\tuple{\mo,\mb,\Tp,\Tn\cup\setof{Q}}$, $\dnx{X}$ those that are no diagnoses for $\tuple{\mo,\mb,\Tp\cup\setof{Q},\Tn}$, and $\dz{X}$ all remaining ones. 
%
A partition $\Pt$ of $\mD$ into three sets is called \emph{q-partition} iff there is a query $Q$ wrt.\ $\mD$ such that $\Pt = \tuple{\dx{Q},\dnx{Q},\dz{Q}}$.
According to the definition of a query, it holds that $Q$ is a query iff both $\dx{Q}$ and $\dnx{Q}$ are non-empty sets. The notion of a q-partition is leveraged by current approaches for \emph{query generation} \cite{DBLP:journals/corr/rodler17jair}, \emph{query verification} \cite{Shchekotykhin2012} and \emph{query quality estimation} \cite{rodler-dx17,rodler17dx_activelearning}.

\begin{example}\label{ex:query_QP}
	Let the known set of diagnoses for $\dpi_{ex}$ be $\mD = \setof{\md_1,\md_2,\md_3}$ (see Example~\ref{ex:diagnoses}).	One query wrt.\ $\mD$ is, e.g., $Q_1 := \setof{\mathit{ActiveResearcher} \sqsubseteq\mathit{Author}}$. Because, \emph{(i)}~adding $Q_1$ to $\Tp$ yields that the removal of $\md_1$ or $\md_2$ from $\mo$ no longer breaks the unwanted entailment $\mathit{Employee}(\mathit{ann})$, i.e., $\md_1,\md_2$ are no longer minimal diagnoses, \emph{(ii)}~moving $Q_1$ to $\Tn$ means that $\md_3$ is not a minimal diagnosis anymore, as, to prevent the entailment of (the new negative test case) $Q_1$, at least one of $\tax_1,\tax_2$ must be deleted. 
	The resulting q-partition for $Q_1$ is thus $\langle\dx{Q_1},\dnx{Q_1},\dz{Q_1}\rangle = \tuple{\setof{\md_3},\setof{\md_1,\md_2},\emptyset}$. 
	Note, e.g., $Q_2 := \setof{\mathit{Author} \sqsubseteq \mathit{Person}}$, 
	is not a query since no diagnosis in $\mD$ is invalidated upon assigning $Q_2$ to $\Tp$, i.e., a positive answer does not give any useful information for diagnoses discrimination. Intuitively, this is because $Q_2$ does not contribute to the violation of $\tn_1$ (in fact, the other ``part'' $\mathit{Author} \sqsubseteq \mathit{Employee}$ of $\tax_3$ does so).
	\qed
\end{example}

\noindent\textbf{Problem Definition.} 
The 
query-based ontology fault localization problem (QFL) is to find for an FPI a series of questions to an expert, the answers of which lead to a single possible remaining fault assumption. The optimization version of the problem includes the additional goal to minimize the effort of the expert. Formally: 
\begin{problem}[(Optimal) QFL]\label{prob:opt_QFL}
	\textbf{Given:} FPI $\tuple{\mo,\mb,\Tp,\Tn}$.
	\textbf{Find:} (Minimal-cost) series of queries $Q_1,\dots,Q_k$ s.t.\ there is only one minimal diagnosis for $\langle\mo,\mb,\Tp\cup\Tp'$, $\Tn\cup\Tn'\rangle$ where $\Tp'$ ($\Tn'$) is the set of all positively (negatively) answered queries, i.e., $\Tp':=\{Q_i\mid 1\leq i \leq k, \oracle(Q_i)=y\}$ and $\Tn':=\{Q_i\mid 1\leq i \leq k, \oracle(Q_i)=n\}$.
\end{problem}
Note, there is no unified definition of the cost of a solution to the QFL problem. Basically, any function mapping $Q_1,\dots,Q_k$ to a non-negative real number is possible. We pick up on this discussion again in Sec.~\ref{sec:discussion_of_existing_approaches}. 
\begin{example}\label{ex:TOD_problem}
	Let the actual diagnosis be $\md_3$, i.e., $\tax_3$ is the (only) faulty axiom in $\mo$ (intuition: an author is not necessarily employed, but might be, e.g, a freelancer). Then, given $\dpi_{ex}$ as an input, solutions to Problem~\ref{prob:opt_QFL}, yielding the final diagnosis $\md_3$, are, e.g., $\Tp'=\emptyset,\Tn'=\setof{\setof{\exists \mathit{writes}.\top \sqsubseteq \mathit{Employee}},\setof{\mathit{Author} \sqsubseteq \mathit{Employee}}}$ or $\Tp'=\setof{\setof{\mathit{ActiveResearcher} \sqsubseteq\mathit{Author}}},\Tn'=\emptyset$.
	Measuring the querying cost by the number of queries, the latter solution (cost: 1) is optimal, the former (cost: 2) not.  
	\qed
\end{example}

\section{Discussion of Query-based Fault Localization Approaches}
\label{sec:discussion_of_existing_approaches}
In this section we analyze existing approaches regarding the assumptions they make about (the query answering behavior of) the interacting user, 
their properties resulting from natural design choices,
 as well as optimization criteria they consider.

\noindent\textbf{Assumptions about Query Answering.} 
All proposed approaches drawing on the interactive methodology described in Sec.~\ref{sec:basics} make the assumption \emph{during their computations and optimizations} that the expert evaluates each query as a whole. 
That is, they perform an assessment of the query effect or (information) gain \emph{based on two possible outcomes} ($y$ and $n$). 
However, in fact, since queries might contain multiple axioms, the feedback of an expert to a query might take a multitude of different shapes. Because, the expert might not view the query as an atomic question, but at the axiom level, i.e., inspecting axioms one-by-one. Clearly, to answer the query $Q=\setof{\tax_1,\dots,\tax_m}$ positively---i.e., that the conjunction of the axioms $\tax_1,\dots,\tax_m$ is an entailment of the intended ontology---one needs to scrutinize and approve the entailment of all single axioms. To negate the query $Q$, in contrast, 
it suffices
to detect one of the $m$ axioms in $Q$ which is not an entailment of the intended ontology. In this latter case, however, we might reasonably assume the interacting expert to be able to name (at least this) one \emph{specific} axiom $\tax^*\in Q$ that is not an intended entailment. We might think of $\tax^*$ as a ``witness of the falsehood of the query''. This additional information---beyond the mere negative answer $n$ indicating that some \emph{undefined} query axiom must not be entailed---justifies the addition of $n^*:=\setof{\tax^*}$, instead of $Q$, to the negative test cases. Please note that $n^*$ provides stronger information than $Q$, and thus potentially rules out more diagnoses. The reason is that each diagnosis that entails $Q$ (i.e., is invalidated given the negative test case $Q$) particularly entails $\tax^*$ (i.e., is definitely invalidated given the negative test case $n^*$). Apart from the scenario where experts provide just a falsehood-witness in the negative case, they might give even more information. For instance, an expert could walk through the query axioms until either a non-entailed one is found or all axioms have been verified as intended entailments. In this case, there might as well be some entailed axioms 
encountered before the first 
non-entailed
one is detected. The set of these entailed axioms could then be added to the positive test cases---in addition to the negative test case $n^*$. Alternatively, the expert might also continue evaluating axioms after recognizing the first non-entailed axiom $\tax^*$, in this vein providing the classification of all single query axioms in $Q$.

Based on this discussion, we might---besides the \emph{query-based} expert that answers queries as a whole, exactly as specified by the $\oracle$ function defined in Sec.~\ref{sec:basics}---characterize (at least) three different types of \emph{axiom-based} experts which supply information beyond the mere 
$n$ label of the query in the negative case:\footnote{Note that a positive query answer ($y$) implicitly provides \emph{axiom-level} information, i.e., the positive classification of all query-axioms. Therefore, the discussed expert types differ only in their query negation behavior.} 
\begin{itemize}
	\item \emph{Minimalist:} Provides exactly one $\tax^* \in Q$ which is not entailed by the intended ontology.
	\item \emph{Pragmatist:} Provides the first 
	found axiom $\tax^* \in Q$ that is not entailed by the intended ontology, 
	and all axioms evaluated as entailments of the intended ontology until $\tax^*$ was found.
	\item \emph{Maximalist:} Provides the classification of each axiom in $Q$ as either an entailment or a non-entailment of the intended ontology. 
\end{itemize}

Consequently: \emph{(i)} In general, without knowing the answering type of the interacting expert in advance, the binary query evaluation conducted in existing works is only an approximation. \emph{(ii)} Also if the expert type is known, 
it is an open issue which form of interaction can
exploit the expert knowledge most beneficially and economically. 
Our experimental evaluations reported in Sec.~\ref{sec:eval} shall confirm (i) and bring light to (ii).

%

\noindent\textbf{Natural Design Choices.}
As explicated in Sec.~\ref{sec:basics}, the principle behind queries is the comparison of entailments and non-entailments resulting from different fault assumptions (diagnoses). In existing works \cite{Shchekotykhin2012,Rodler2013}, this is often done by computing common entailments (of specific types)---e.g., subsumption and assertion axioms resulting from classification and realization reasoning services \cite{DLHandbook}---for some diagnoses and verify whether some other diagnosis becomes inconsistent when assuming correct these axioms.
At this, it stands to reason to use and further process \emph{all} entailments returned by the reasoner. Moreover, the fewer entailments are used, the higher is the chance that these are entailed by all (known) diagnoses and hence do not constitute a query. Besides, assuming a \emph{query-based} expert (see above), query selection heuristics \cite{Shchekotykhin2012,Rodler2013,rodler17dx_activelearning,rodler2018ruleML} can be optimized to a higher degree due to the simple fact that a larger allowed cardinality of queries implies a larger search space for queries. 
For these reasons, it is quite natural to specify queries as \emph{sets of} axioms.

\noindent\textbf{Optimization Criteria.}
The meaning of ``minimal-cost'' in Problem~\ref{prob:opt_QFL} might be defined in different ways. Most existing works on query-based fault localization, e.g., \cite{Shchekotykhin2012,schekotihin2018ontodebug,Rodler2013,Rodler2015phd}---
especially in the empirical analyses they present---specify the cost of a solution $Q_1,\dots,Q_k$ to the QFL problem to be \emph{the number of} queries, i.e., $k$. The underlying assumption in this case is that each two queries mean the same (answering) cost for an expert. Given that queries might include fewer or more axioms of lower or higher (syntactic or semantic) complexity, we argue that this cost measure might be too coarse-grained to capture the effort for an interacting expert in a realistic way. Instead, it might be better suited to measure the costs at the axiom level. However, a fundamental problem with a minimization of the axiom level costs is the need to compute the specific query axioms for multiple (or all) queries, which generally involves high computation costs in terms of a high number of reasoner calls. A remedy to this problem and a two-staged technique to minimize both the number of queries and the costs at the axiom level is suggested by \cite{DBLP:journals/corr/rodler17jair}. However, the user type taken as a basis for these optimizations is again the query-based one (see above).
%
%
%

\section{New Approach to Expert Interaction}
\label{sec:new_approach}
\subsection{Idea}
\label{sec:new_approach:idea}
In the light of the issues pointed out in Sec.~\ref{sec:discussion_of_existing_approaches} and following quite straightforward from the given argumentation, we propose a new way of expert interaction for fault localization in ontologies, namely to abandon ``batch-queries'' including multiple axioms and to focus on so-called \emph{singleton queries} instead. That is, we suggest to restrict queries to only single-axiom questions. Formally:
\begin{definition}[Singleton Query]\label{def:singleton_query}
	Let $\mD$ be a set of diagnoses for an FPI $\tuple{\mo,\mb,\Tp,\Tn}$. Then, $Q$ is a singleton query (wrt.\ $\mD$) iff $Q$ is a query (wrt.\ $\mD$) and $|Q|=1$.
\end{definition}
\subsection{Properties}
\label{sec:new_approach:properties}
The \emph{advantages} of singleton queries are the following:
\begin{itemize}
	\item \emph{Maximally-fine granularity of optimization loop:} Each atomic expert input (i.e., each classified axiom) can be directly taken into account to optimize further computations and expert interactions. Simply put, each axiom the expert is asked to classify is a function of \emph{all} so-far classified axioms.
	\item \emph{Smaller search space:} There are fewer singleton queries than there are general queries. Therefore, the worst-case search costs are lower for singleton queries.
	\item \emph{Realistic query assessment:} For singleton queries, the binary-outcome assessment performed by the discussed approaches is exact, plausible and not just an approximation of the possible real cases---independent of the expert (type). The reason is that there \emph{are} exactly two possible outcomes, namely $y$ (query axiom added to $\Tp$) and $n$ (query axiom added to $\Tn$). 
	\item \emph{Direct re-use of existing works:} 
	Concepts (e.g., heuristics) and techniques (e.g., search algorithms) defined for queries can be immediately re-used for singleton queries, because each singleton query \emph{is} a (specific) query. 
	\item \emph{Unique optimization criterion:} Query-number minimization and (axiom-based) ans-wering-cost minimization coincide for singleton queries. This unifies the two competing and arguable views on the query optimization problem.
	\item \emph{More informative expert feedback:} Negative answers to singleton queries provide more information than negative answers to normal queries as the former imply that we know one axiom which is wrong \emph{for sure}, whereas the latter just tell us that \emph{one of a set of} axioms is not true. Therefore, singleton queries, by their nature, implicitly appoint how they are answered, independent of the expert (type). Because all discussed expert types coincide for singleton queries.
\end{itemize}
On the downside, the smaller search space---apart from the better worst-case query optimization complexity---can be seen as a \emph{disadvantage} as well. Because soundness of the query search is more difficult to obtain, i.e., more considerations and computations than for normal queries are required to ensure that the search outcome is indeed a \emph{singleton} query. 
For instance, after having optimized a predefined heuristic measure for some query candidate (set of axioms) to a sufficient degree, existing approaches \cite{Shchekotykhin2012,DBLP:journals/corr/Rodler16a} post-process this candidate by a query-size minimization step. This step, however, does not guarantee the reduction to a single axiom. Thus, beside all the mentioned advantages of singleton queries, an algorithmic and computational challenge towards their efficient generation and optimization remains to be solved. 
\subsection{Generation and Optimization} 
\label{sec:new_approach:generation+optimization}
As a first step in this direction we suggest an algorithm that, 
given a set of diagnoses $\mD$, finds the (next) heuristically-optimal\footnote{The \emph{global} optimization of query costs is proven NP-hard \cite{hyafil1976} (even without considering the reasoning complexity for diagnosis and query generation). Hence, the best that methods can achieve is to optimize some heuristic in each query computation iteration. To this end, a one-step-lookahead query evaluation \cite{dekleer1992onesteplookahead} (what is the expected situation after the query has been answered?) is state-of-the-art and also used in this present as well as in existing works. Note the similarity to decision tree learning approaches \cite{quinlan1986induction}.} 
singleton query $Q \subseteq \mo$ (wrt.\ $\mD$) to ask the expert. 
In this vein, the algorithm can be used in each iteration of a sequential fault localization session. Such a session is characterized by a loop involving a re-iteration of the three phases \emph{(1)}~fault hypotheses generation (computation of diagnoses), \emph{(2)}~query generation and optimization, and \emph{(3)}~query answering and incorporation of the newly acquired test case(s), until only one diagnosis is left.\footnote{Note that this condition must be fulfilled after having obtained the answer to a \emph{finite} number of queries as each query, regardless of its answer, rules out at least one diagnosis (cf.\ Sec.~\ref{sec:basics}), and the number of diagnoses is bounded by the number of subsets of the \emph{finite} ontology $\mo$.} 
By the theory of model-based diagnosis \cite{Reiter87,dekleer1987}, this final diagnosis necessarily includes the faulty axioms explaining all observed problems (e.g., inconsistency, unsatisfiable classes, wrong entailments) of the ontology.
%
%
%
Thus, used for query computation in a sequential session, our algorithm presented below will deliver a (heuristics-based approximation of the optimal) 
series of ontology axioms such that the assignment of each of these axioms to either the positive or the negative test cases solves Problem~\ref{prob:opt_QFL}.

The works of \cite{DBLP:journals/corr/Rodler16a,rodler-dx17} serve as a theoretical and algorithmic basis for our method. In fact, we slightly extend the theory and 
adapt the algorithm presented there to accommodate singleton queries. First, we briefly review the existing query computation and optimization algorithm for normal queries, and next we present our adaptations to it.
%
%

\noindent\textbf{Query Computation and Optimization for Normal Queries (Recap).} 
Basically, the algorithm \cite{DBLP:journals/corr/rodler17jair} is subdivided into two stages, namely a search for a heuristically-optimal q-partition $\Pt$ (stage~1) and a search for a cost-optimal query (set of axioms) for this fixed q-partition $\Pt$ (stage~2). At this, the first stage serves the purpose of optimizing a heuristic function, e.g., the expected information gain \cite{dekleer1987,Shchekotykhin2012}, that aims at minimizing the expected \emph{number of queries}. 
The goal of the second stage is to
minimize the \emph{cost for query answering} based on some axiom-based cost measure, e.g., the number of axioms.

\paragraph{Stage~1:} 
Here, a heuristic search
is performed. Such a search is characterized \cite{russellnorvig2016} by a start state, a goal state, a successor function (what are the immediate neighbor states of a given state?) as well as a heuristic function (what is the expected utility of visiting a given state?). Originally, the ``depth-
first, local best-first backtracking'' algorithm works as follows. 
\emph{(Depth-first):} Starting from the initial partition $\tuple{\emptyset,\mD,\emptyset}$ (start state), the
search proceeds downwards by ``shifting'' diagnoses from the middle ($\dnx{}$) to the left ($\dx{}$) part of the q-partition\footnote{Note, q-partitions with non-empty $\dz{}$ (i.e., right) part tend to be unfavorable (see argumentation in \cite{DBLP:journals/corr/rodler17jair})
and are thus totally neglected in the q-partition search discussed here for efficiency reasons. So, in the sequel, we will always assume $\dz{} = \emptyset$ for all mentioned q-partitions.} until \emph{(a)}~a q-partition with sufficiently optimal heuristic value 
has been found (goal state),  or
\emph{(b)}~there are no successors of the currently analyzed q-partition.
\emph{(Local best-first):} At each current q-partition, the focus moves on to the best \emph{direct} successor q-partition, according to the given heuristic function.\footnote{The predicate ``local'' refers to the fact that the best q-partition to visit next is determined \emph{solely} based on the direct successors of the q-partition.}
\emph{(Backtracking):} The search procedure backtracks in case all successors of a q-partition have been explored and no goal q-partition has been found yet.
In this case, the next-best unexplored sibling 
of the q-partition will be analyzed next. 

The detailed definition of the used successor function is beyond the scope of this work. Therefore, we exemplify the underlying principle through an example \cite{DBLP:journals/corr/rodler17jair}:\footnote{In the sequel, we will use the following abbreviations: Given a collection of sets $C$, we denote by $U(C)$ the union and by $I(C)$ the intersection of all sets in $C$.}
\begin{example}\label{ex:q-partition_search}
Let a set of minimal diagnoses for an FPI be 
$\mD = \{\md_1,\md_2,\md_3, \md_4,\md_5$, $\md_6\} = \{\setof{2,3},\setof{2,5},\setof{2,6},\setof{2,7},\setof{1,4,7},\setof{3,4,7}\}$, where axioms are represented as 
numbers for simplicity of notation. Be the current q-partition analyzed in the search $\Pt = \tuple{\setof{\md_5},\mD \setminus\setof{\md_5},\emptyset}$. Given a q-partition as an input, the goal of the successor function is to output the set of all q-partitions obtainable by \emph{minimal} changes from the input q-partition. These direct successor q-partitions 
can be computed by means of the notion of a trait. The \emph{traits} for a q-partition $\tuple{\dx{},\dnx{},\emptyset}$ are given by $\md'_i:=\md_i \setminus U(\dx{})$ for all $\md_i \in \dnx{}$. For $\Pt$, the traits $\md'_1,\md'_2,\md'_3,\md'_4,\md'_6$ are given by $\setof{2,3}$, $\setof{2,5}$, $\setof{2,6}$, $\setof{2}$, $\setof{3}$, where, e.g., $\md'_6 = \md_6\setminus U(\setof{\md_5}) = \setof{3,4,7}\setminus\setof{1,4,7} = \setof{3}$. 
\emph{Successors of a q-partition exist iff there are at least two different subset-minimal traits for this q-partition.} For $\Pt$, this holds true, since $\md'_4$ as well as $\md'_6$ are subset-minimal; note, however, that all other traits are not subset-minimal as they are each proper supersets of $\md'_4$ or $\md'_6$. \emph{If successors exist for a q-partition $\Pt_r=\tuple{\dx{r},\dnx{r},\emptyset}$, then its direct successors are given by the q-partitions resulting from $\Pt_r$ by transferring all diagnoses from $\dnx{r}$ to $\dx{r}$ which have the same trait and whose trait is subset-minimal among all traits for $\Pt_r$.} For $\Pt$, this means that there are two direct successors, namely $\tuple{\setof{\md_5,\md_4},\mD\setminus\setof{\md_5,\md_4},\emptyset}$ and $\tuple{\setof{\md_5,\md_6},\mD\setminus\setof{\md_5,\md_6},\emptyset}$.
\qed
\end{example}  
%
\paragraph{Stage~2:} 
In this phase, a query (set of axioms) is sought for the fixed (and already optimal) q-partition returned by stage 1. \cite{DBLP:journals/corr/rodler17jair} shows that the queries (comprising ontology axioms) for a q-partition are exactly the hitting sets\footnote{A set $H$ is a \emph{hitting set} of a collection of sets $C = \setof{S_1,\dots,S_n}$ iff $H \subseteq S_1 \cup \dots \cup S_n$ and $S_i \cap H \neq \emptyset$ for all $S_i \in C$.} of all traits for this q-partition. Axiom costs can be minimized by computing hitting sets in best-first order, e.g., by means of the hitting set algorithm presented in \cite{Rodler2015phd}. For instance, in order to minimize the number of axioms in the query, a minimum-cardinality-first hitting set computation will do.
\begin{example}\label{ex:query_search}
For the q-partition $\Pt$ from Example~\ref{ex:q-partition_search}, all subsets of $\setof{2,3,5,6}$ that include $2$ or $3$ are queries. The queries with a minimal number of axioms are $\setof{2}$ and $\setof{3}$. 
\qed
\end{example}

\noindent\textbf{Extension to Singleton Queries.}
We now present the amendments to 
the reviewed query computation and optimization algorithm (stages 1 and 2) that are necessary to deal with singleton queries.

To restrict the q-partition search in stage~1 to only q-partitions for singleton queries, we first need a criterion that tells us for which q-partitions associated singleton queries do and do not exist. The following theorem provides such a criterion. The idea is that a singleton query (consisting of an ontology axiom) exists for a q-partition iff all traits for this q-partition include this axiom.
\begin{theorem}[Singleton Query Criterion]\label{theorem:singleton_query_criterion}
Let $\mD$ be a set of minimal diagnoses for the FPI $\tuple{\mo,\mb,\Tp,\Tn}$ and $\tax\in\mo$. Then, $\setof{\tax}$ is a singleton query (wrt.\ $\mD$) iff there is a q-partition $\Pt = \tuple{\dx{},\dnx{},\emptyset}$ (wrt.\ $\mD$) such that  $I(\dnx{}) \setminus U(\dx{}) \supseteq \setof{\tax}$.
\end{theorem}
Note that Theorem~\ref{theorem:singleton_query_criterion}, in particular, means that each axiom occurring in some, but not all, (known) diagnoses in $\mD$ is a singleton query. However, we want to systematically enumerate an as small as possible number of such queries in a (heuristically) optimal order. Therefore, we next ``translate'' the above criterion to a successor function that, for any given q-partition, generates all and only singleton query successor q-partitions. 
Such a function, plugged into the search (stage 1) described above instead of the successor function for normal queries---while re-using everything else of the existing algorithm---yields a sound and complete method for singleton query q-partitions.
\begin{example}\label{ex:singleton_criterion}
Recall the diagnoses set $\mD$ from Example~\ref{ex:q-partition_search}. For this, e.g., $\setof{7}$ is a singleton query as there is the q-partition $\Pt:=\tuple{\setof{\md_1,\md_2,\md_3},\setof{\md_4,\md_5,\md_6},\emptyset}$ for which the criterion $I(\dnx{})\setminus U(\dx{}) = \setof{7}\setminus\setof{2,3,5,6} \supseteq \setof{7}$ holds. However, assuming $\mD$ consisted only of, e.g., $\md_4,\md_5,\md_6$, $\setof{7}$ would not be a (singleton) query (wrt.\ $\mD$). The reason is that a negative answer to it would not invalidate any (known) diagnosis.
	\qed
\end{example}
%
The following matrix-representation for a q-partition's traits is a useful tool towards defining the successor function for singleton query q-partitions.
%
\begin{definition}[Axioms-Traits Matrix (ATM)]\label{def:ax_diag_matrix}
	Let $\Pt = \tuple{\dx{},\dnx{},\emptyset}$ be a q-partition where $\dnx{} = \setof{\md_{k_1},\dots,\md_{k_n}}$
	and $\setof{\tax_1,\dots,\tax_m}$ be the set of all axioms occurring in the traits $\md'_{k_1},\dots,\md'_{k_n}$ for $\Pt$.
	Then, we call the $m\times n$-matrix $A_\Pt = (a_{ij})$, where $a_{ij} = 1$ iff $\tax_i \in \md'_{k_j}$ and $a_{ij} = 0$ else, the \emph{axioms-traits matrix (ATM)} for $\Pt$.
\end{definition}
%
\begin{example}\label{ex:ATM}
For the q-partition mentioned in Example~\ref{ex:singleton_criterion}, the ATM is given by the following matrix. In fact, the matrix represents the statements that axiom $1 \in \md'_5$ (first row), axiom $4$ is an element of $\md'_5$, $\md'_6$ (second row), and so on. \qed
\vspace{-18pt}
\begin{align*}
\begin{blockarray}{cccc}
\scriptstyle \md_4 & \scriptstyle \md_5 & \scriptstyle \md_6 & \\
\begin{block}{(ccc)c}
	 0 & 1 & 0 & \;\; \scriptstyle 1 &\\
	 0 & 1 & 1 & \;\; \scriptstyle 4 &\\
	 1 & 1 & 1 & \;\; \scriptstyle 7 & \\
\end{block}
\end{blockarray}
\end{align*}
%

\end{example}
\begin{definition}[Domination]
	Let $A_\Pt$ be the $m\times n$ ATM for a q-partition $\Pt$ and $a_{i.}$ as well as $a_{j.}$ be matrix rows where $1\leq i,j \leq m$. Then, $a_{i.}$ \emph{dominates} $a_{j.}$ iff $a_{ir}=1$ for all indices $r \in \{1,\dots,n\}$ for which $a_{jr}=1$. Further, $a_{i.}$ \emph{strictly dominates} $a_{j.}$ iff $a_{i.}$ dominates $a_{j.}$, but $a_{j.}$ does not dominate $a_{i.}$. We call a row \emph{superior row} iff it includes at least one 0-entry and is not strictly dominated by any other row.
\end{definition}
\begin{example}\label{ex:superior_rows}
In the ATM given in Example~\ref{ex:ATM}, the second row is the only superior row. The first row is not superior because it includes only 1-entries, and the last row is not since it is dominated by the second one. 	
	\qed
\end{example}
The next theorem states the successor function for singleton queries. Informally, it says that each superior row of a q-partition's ATM
represents a singleton query successor q-partition of this q-partition. Each diagnosis associated with a $1$-entry in a superior row is an element of the $\dnx{}$ part of the successor q-partition and all remaining diagnoses in $\mD$ are in the $\dx{}$ part.  
\begin{theorem}[Singleton Query Successor Function]\label{theorem:singleton_query_succ_function}
Let $\mD$ be a set of minimal diagnoses for an FPI and $\Pt = \tuple{\dx{},\dnx{},\emptyset}$ be a q-partition (wrt.\ $\mD$). 
Let further $A_\Pt$ be the ATM associated with $\Pt$, and $R$ be the set of the row indices of all superior rows in $A_\Pt$. 
Then, the direct singleton query successors of $\Pt$ are given by $\setof{\tuple{\dx{i},\dnx{i},\emptyset}\mid a_{i.} \in R}$ where $\dnx{i} = \setof{\md_{k_j} \mid a_{ij}=1}$ and $\dx{i}=\mD\setminus\dnx{i}$. 
\end{theorem}
\begin{example}
Let us reconsider the q-partition $\Pt$ of Example~\ref{ex:singleton_criterion}. Using Theorem~\ref{theorem:singleton_query_succ_function} and our observations of Examples~\ref{ex:ATM} and \ref{ex:superior_rows}, we find that $\tuple{\setof{\md_1,\md_2,\md_3,\md_4},\setof{\md_5,\md_6},\emptyset}$ is the only singleton query successor q-partition of $\Pt$.
	\qed
\end{example}
For stage~2 we get---immediately from Theorem~\ref{theorem:singleton_query_criterion}\footnote{Note, $I(\dnx{}) \setminus U(\dx{})$ is exactly the intersection of all traits of the q-partition $\tuple{\dx{}, \dnx{}, \emptyset}$.}---that each axiom appearing in all traits of the singleton query q-partition selected in stage~1 is a singleton query:
\begin{corollary}[Singleton Query Extraction]\label{cor:singleton_query_extraction_from_q-partition}
	Let $\Pt = \tuple{\dx{}, \dnx{}, \emptyset}$ be a q-partition that satisfies the criteria given by Theorem~\ref{theorem:singleton_query_criterion} and let $A_\Pt$ be the ATM associated with $\Pt$. Then, all singleton queries (consisting of axioms in $\mo$) for $\Pt$... 
	\begin{enumerate}
		\item ...are given by $\setof{\setof{\tax} \mid \tax \in I(\dnx{}) \setminus U(\dx{})}$.
		\item ...are given exactly by the axioms representing rows with only 1-entries in $A_\Pt$.
	\end{enumerate}
\end{corollary}
\begin{example}\label{ex:singleton_query_extraction_from_ATM}
The only singleton query $\setof{\tax}$ for $\tax \in \mo$ for the q-partition $\Pt$ of Example~\ref{ex:singleton_criterion} is $\setof{7}$. This can be seen from $\Pt$'s ATM shown in Example~\ref{ex:ATM} where the row of $7$ is the only row without any 0-entry.\qed
\end{example}  

\begin{algorithm}[t]
	\scriptsize
	\caption{\small(Singleton) Query Selection}\label{algo:singleton_query_selection}
	\begin{algorithmic}[1]
		\Require set of minimal diagnoses $\mD$ for some FPI $\langle\mo,\mb,\Tp,\Tn\rangle$,
		heuristic $h_1$ (to minimize \# of queries) to be optimized in stage~1,
		heuristic $h_2$ (to minimize effort per query) to be optimized in stage~2, 
		boolean $s$ affecting the generation of a singleton ($s=\true$) or a normal ($s=\false$) query 
		\Ensure best (singleton) query wrt.\ $h_2$ among all queries for the q-partition of $\mD$ with best $h_1$
		\State $\Pt \gets \Call{findBestQPartition}{\mD,h_1,s}$ \Comment{stage~1}
		\State $Q \gets \Call{findBestQueryForQPartition}{\Pt,h_2,s}$ \Comment{stage~2}
		\State \Return $Q$
	\end{algorithmic}
	\normalsize
\end{algorithm}

\subsection{Complexity Analysis}
\label{sec:new_approach:complexity}
The complexity of the suggested algorithm for the generation of a heuristically-optimal singleton query for a given sample of diagnoses is as follows:
\begin{theorem}[Complexity]\label{theorem:complexity}
Let $\mD$ be the set of known diagnoses 
and $n_{\max}$ be the number of axioms in the diagnosis of maximal size in $\mD$. Then, Algorithm~\ref{algo:singleton_query_selection} with setting $s=\true$ requires $O(n_{\max}^4|\mD|^3)$ time and $O(n_{\max}|\mD|^3)$ space.
\end{theorem}
\begin{proof} We first consider the time and then the space complexity.
	
\noindent\emph{Time complexity (stage~1):} At each node in the search tree a q-partition and a respective ATM must be computed. The construction of a q-partition requires $O(|\mD|)$ steps. The creation of an ATM needs one iteration through all (axioms of the) diagnoses in $\dnx{}$, i.e., $O(n_{\max}|\mD|)$ steps.

Successor extraction for the q-partition at each node requires the finding of all superior rows in the ATM. This can be accomplished by checking, for each row, whether it has a 0-entry and whether it is not dominated by any other row. There are $O(n_{\max}|\mD|)$ rows (if all diagnoses are disjoint and have equal size $n_{\max}$) and each row has $n_{\max}$ entries. Checking the presence of a 0-entry requires $O(n_{\max})$ checks. Domination can be checked by comparing all (same-indexed) entries of two rows, i.e., by means of $O(n_{\max})$ comparisons. There are $O((n_{\max}|\mD|)^2)$ pairs of rows for the domination test. Hence, we need $O(n_{\max}^2 |\mD| + n_{\max}(n_{\max}|\mD|)^2)= O(n_{\max}^3|\mD|^2)$ steps for successor computation. Altogether, the time complexity at each node is thus in $O(n_{\max}|\mD|+n_{\max}^3|\mD|^2) = O(n_{\max}^3|\mD|^2)$.

As a consequence of Theorem~\ref{theorem:singleton_query_criterion}, the number of explored q-partitions in stage~1 is bounded by $|U(\mD)| \leq \sum_{\md\in\mD} |\md| \leq n_{\max}|\mD|$, i.e., the q-partition search tree has $O(n_{\max}|\mD|)$ nodes.

Consequently, the time complexity of stage~1 is in $O(n_{\max}^4|\mD|^3)$.

\noindent\emph{Time complexity (stage~2):}
For \emph{one} q-partition $\Pt$ (the one selected in stage~1), one (all) singleton queries for $\Pt$ can be extracted by scanning all rows of $\Pt$'s ATM until one (all) row(s) with only 1-entries are found (Corollary~\ref{cor:singleton_query_extraction_from_q-partition}). This can be done in $O(n_{\max}|\mD|)$ steps (one check for each entry of the ATM). Since all singleton queries can be extracted within this time bound, the best query as per some heuristic can in particular.

\noindent\emph{Time complexity (overall):}
So, the time complexity of Algorithm~\ref{algo:singleton_query_selection} (stage 1 and 2 together) is in $O(n_{\max}^4|\mD|^3 + n_{\max}|\mD|) = O(n_{\max}^4|\mD|^3)$.

\noindent\emph{Space complexity (stage~1):}
For each node of the q-partition search tree, we need to store the respective q-partition. The ATM associated with this q-partition needed for successor computation can be computed only at node expansion and does not need to be permanently stored. Also, it can be discarded as soon as all successors have been generated. Note, since the (heuristically-)\emph{best} successor is always chosen as a next node for expansion by the algorithm, such an on-demand computation of the successor q-partitions is not possible. Each q-partition can be stored in $O(n_{\max}|\mD|)$ space (which is the space to store all diagnoses in $\mD$). Any ATM requires $O(n_{\max}|\mD|^2)$ entries because it has at most $n_{\max}|\mD|$ rows (if all diagnoses are disjoint and have equal size $n_{\max}$) and at most $|\mD|$ columns (there can be no more diagnoses in $\dnx{}$ than there are in $\mD$). 
%

Concerning the number of nodes that must be simultaneously stored during the q-partition search,
observe that each successor q-partition results from a q-partition by shifting some diagnosis from its $\dnx{}$ to its $\dx{}$ set. Hence, at most $|\mD|$ successors might exist for any q-partition, i.e., the branching factor of the search tree is bounded by $|\mD|$. Moreover, the depth of the search tree is bounded by $|\mD|$ as well, since along any branch downwards in the search tree diagnoses can only be shifted from $\dnx{}$ to $\dx{}$ (and not vice versa). Since a depth-first search is executed, the space complexity is the product of the branching factor and the maximal tree depth, and is thus given by $O(|\mD|^2)$ search tree nodes. 

Altogether, the space complexity of stage~1 amounts to the space for a q-partition times the number of q-partitions simultaneously in memory, plus the space for a single ATM (of the currently expanded node). Therefore, stage~1 requires $O(n_{\max}|\mD|^3 + n_{\max}|\mD|^2) = O(n_{\max}|\mD|^3)$ space.

\noindent\emph{Space complexity (stage~2):}
No additional amount of storage is required for stage~2 because according to Corollary~\ref{cor:singleton_query_extraction_from_q-partition} the singleton query can be extracted directly from the ATM of the q-partition selected in stage~1, which however must already be in memory.

\noindent\emph{Space complexity (overall):}
The overall space complexity is thus in 
$O(n_{\max}*|\mD|^3)$. \qed
\end{proof}
Two remarks: First, the input size $I$ of Algorithm~\ref{algo:singleton_query_selection} is in $O(n_{\max}|\mD|)$. So, in terms of $I$, the time and space complexity is in $O(I^4)$ and in $O(I^3)$, respectively. Second,   
the number of diagnoses $|\mD|$ cannot grow arbitrarily because it is a predefined fixed number that can be set to any (small) value greater or equal $2$ \cite{Rodler2015phd}. 

\section{Evaluation}
\label{sec:eval}
\noindent\textbf{Goal.} The aim of the following experiments is the analysis of normal queries under different answering conditions (expert types discussed in Sec.~\ref{sec:discussion_of_existing_approaches}) and the comparison between normal queries and the proposed singleton queries. Focus of the investigations is the \emph{required effort for the expert} for fault localization and the \emph{query computation time}. Particular questions of interest are:
\begin{enumerate}[label=Q\arabic{*}, ref=Q\arabic{*}, leftmargin=2.0em]
	\item \label{Q1} Since existing methods compute and optimize queries based on the assumption of a \emph{query-based expert} (cf.\ Sec.~\ref{sec:discussion_of_existing_approaches}), which implications does a violation of this assumption have on the efficiency of fault localization?
	\item \label{Q2} Given (a system that computes) a particular type of query, which answering strategy to recommend the interacting expert to pursue?
	\item \label{Q3} Given a particular (type of) expert, which type of queries to ask them?
	\item \label{Q4} What is the expected waiting time for the next query in all scenarios?
	\item \label{Q5} What is better overall, normal or singleton queries?
\end{enumerate}

\setlength{\tabcolsep}{4pt}
\begin{table}[t]
	\renewcommand\arraystretch{1.2}
	\scriptsize
	\centering
	\caption{\small Dataset used in the experiments.}
	\label{tab:dataset}
	\begin{minipage}{0.5\linewidth}
		\begin{tabular}{@{}llrlr@{\kern3pt}} 
			\toprule
			$j$ & ontology $\mo_j$				& $|\mo_j|$& expressivity \textsuperscript{\textbf{1)}} 		& \#D/min/max \textsuperscript{\textbf{2)}} \\ \midrule
			1 & Koala (K) \textsuperscript{\textbf{3)}}			& 42 		& $\mathcal{ALCON}^{(D)}$& 10/1/3     \\
			2 & University (U) \textsuperscript{\textbf{4)}}   		 & 50 		& $\mathcal{SOIN}^{(D)}$& 90/3/4      \\		
			3 & MiniTambis (M) \textsuperscript{\textbf{4)}}			& 173 		& $\mathcal{ALCN}$ 		& 48/3/3	  \\
			4 & CMT-Conftool (CC) \textsuperscript{\textbf{5)}}			& 458 		& $\mathcal{SIN}^{(D)}$& 934/2/16     \\
			5 & Conftool-EKAW (CE) \textsuperscript{\textbf{5)}}			& 491 		& $\mathcal{ALCH}^{(D)}$& 953/3/10     \\
			6 & Transportation (T) \textsuperscript{\textbf{4)}}		& 1300 		& $\mathcal{ALCH}^{(D)}$& 1782/6/9	  \\
			7 & Economy (E) \textsuperscript{\textbf{4)}}			& 1781 		& $\mathcal{ALCH}^{(D)}$& 864/4/8     \\
			8 & DBpedia (D) \textsuperscript{\textbf{6)}}			& 7228 		& $\mathcal{ALCHF}^{(D)}$& 7/1/1     \\
			9 & Opengalen (O) \textsuperscript{\textbf{7)}}			& 9664		& $\mathcal{ALEHIF}^{(D)}$& 110/2/6     \\
			10 & Cton (C) \textsuperscript{\textbf{7)}}			& 33203		& $\mathcal{SHF}$& 15/1/5     \\
			\bottomrule
		\end{tabular}
	\end{minipage}
	\hfill\hspace{25pt}
	\begin{minipage}{0.37\linewidth}
		\setlength{\tabcolsep}{0pt}
		\begin{tabular}{@{}lp{3.8cm}@{}}
			\textbf{Key:} & \\
			\textbf{1):} & Description Logic expressivity \cite{DLHandbook}
			\\
			\textbf{2):} & \#D, min, max denote the number, the min.\ and max.\ size of minimal diagnoses for the input FPI. 
			\\
			\textbf{3):} & Ontology included in the Prot\'eg\'e Project for educational purposes.\\
			\textbf{4):} & Sufficiently complex FPIs (\#D $\geq 40$) used in \cite{Shchekotykhin2012}.\\
			\textbf{5):} & Hardest FPIs mentioned in \cite{stuckenschmidt2008debugging}.\\
			\textbf{6):} & Faulty version of the DB-Pedia ontology, downloaded from \url{}.\\
			\textbf{7):} & Hardest FPIs tested in \cite{Shchekotykhin2012}.
		\end{tabular}
	\end{minipage}
\end{table}

\noindent\textbf{Dataset, Experiment Settings and Measurements.}
The dataset of ontologies used in the experiments is given in Tab.~\ref{tab:dataset}. All ontologies are real-world examples and are inconsistent and/or incoherent. Each of the ontologies $\mo$ was used to specify an FPI $\dpi := \tuple{\mo,\emptyset,\emptyset,\emptyset}$, i.e., the background knowledge $\mb$ as well as the positive ($\Tp$) and negative ($\Tn$) test cases were (initially) empty. Tab.~\ref{tab:dataset} also shows the \emph{diagnostic structure} (\# of axioms $|\mo|$, logical expressivity, \# and min./max.\ size of minimal diagnoses) for the considered FPIs.
As heuristics ($h_1$) for stage~1 we used the query selection measures discussed in \cite{rodler17dx_activelearning,rodler2018ruleML}. For stage~2 we used the number of axioms in the query as a heuristic ($h_2$).
For each FPI and each heuristic $h_1$ we ran $20$ fault localization sessions (each time using a different random specification of the actual diagnosis to be located). The number of diagnoses computed before each query selection (i.e., given as input to Algorithm~\ref{algo:singleton_query_selection}) was set to (maximally) $|\mD|=10$. Since some heuristics ($h_1$) depend on the diagnoses probabilities, we sampled and assigned uniform random probabilities to diagnoses for each FPI.
For each performed fault localization session we measured 
\begin{enumerate}[label=M\arabic{*}, ref=M\arabic{*}, leftmargin=2.0em]
	\item \label{obs:time} the average computation time to find the best next query  (\emph{time per Q}), 
	\item \label{obs:QPs} the average number of q-partitions generated per computed 
	query (\emph{generated QPs per Q}), and
	\item \label{obs:numQ} the number of answered queries (\emph{\#Q}) as well as
	\item \label{obs:numAx} the number of classified query-axioms (\emph{\#Ax})
\end{enumerate}
required until finding the predefined actual diagnosis with certainty.

\noindent\textbf{Representation of Experiment Results.}
Each of the Figures~\ref{fig:U_plot} -- \ref{fig:CC_plot} provides a per-ontology overview of the observations regarding \ref{obs:time} -- \ref{obs:numAx} we made throughout the experiments, for the ontologies given in Table~\ref{tab:dataset}. 
Specifically, the bars show 
\ref{obs:numQ} and \ref{obs:numAx} 
for the different expert types, i.e., the minimalist (min), the pragmatist (prag), the maximalist (max), and the query-based expert (q-based), as discussed in Sec.~\ref{sec:discussion_of_existing_approaches}. Moreover, the lines report \ref{obs:time} (red line) and \ref{obs:QPs} (black line). On the x-axis, we have a block showing the values for normal queries (normal Q, left) and a block depicting the measurements for singleton queries (singleton Q, right).  
In order to not overload the figures and because the observations regarding other heuristics are mostly consistent with the presented ones, 
Figures~\ref{fig:U_plot} -- \ref{fig:CC_plot} plot only the results for the most-popular heuristics $h_1$ in the field \cite{Shchekotykhin2012,Rodler2013}, i.e., ENT (maximize information gain per query), SPL (maximize worst-case diagnoses elimination rate per query) and RIO (optimize balance between ENT and SPL). 

Figures~\ref{fig:vioplot_ENT} -- \ref{fig:vioplot_EMCb} give violin plots for all\footnote{Note that the \textsf{MPS} heuristic is not (directly) applicable to singleton queries and thus omitted.} heuristics in the field \cite{rodler17dx_activelearning,rodler2018ruleML} that show the difference in query answering effort 
(\ref{obs:numAx}) between a usage of the best answering strategy for normal queries and the usage of singleton queries. Each violin plot combines a box-plot with a kernel density estimation. In particular, the median is represented by a white dot. If the latter is above (below) the red zero-line, then this means that singleton queries 
imply less (more) expert effort 
in the majority of the observed diagnostic sessions. Simply put, the singleton query approach wins on average iff the white dot is above the red line.
The additional heuristics not mentioned above that are shown in Figures~\ref{fig:U_plot} -- \ref{fig:CC_plot} are RND (random query selection), BME (select query that maximizes the number of diagnoses that can be eliminated with a probability larger than 0.5), KL (select query with maximal information-theoretic ``disagreement'' between query-outcome predictions of the known diagnoses) and EMCb (select query that maximizes expected diagnoses elimination rate). For details on these heuristics see \cite{DBLP:journals/corr/Rodler16a,rodler17dx_activelearning,rodler2018ruleML}.

\noindent\textbf{Discussion of Experiment Results.}
We address questions \ref{Q1} -- \ref{Q5} in turn. 
\vspace{5pt}

\noindent \emph{Ad \ref{Q1}:} 
As shown by the vertical bars in Figures~\ref{fig:U_plot} -- \ref{fig:CC_plot},
if normal queries are answered by an \emph{axiom-based} strategy (min, prag, or max) that provides labels for (some or all) axioms in the query, then the effort for the expert is significantly lower than in case of a \emph{query-based} strategy where an expert just gives a label for the query as such. This effort reduction holds for all ontologies and in terms of both the number of queries (\#Q) and the number of checked axioms (\#Ax). In fact, this result is not very surprising. The simple reason for it is that each axiom-based method involves strictly more informative answers than a query-based answering style (cf.\ Sec.~\ref{sec:discussion_of_existing_approaches}).    
However, not all of the axiom-based approaches are equally good, as analyzed in \ref{Q2}. 
\vspace{5pt}

\noindent \emph{Ad \ref{Q2}:} \emph{(Normal queries:)} As all of the Figures~\ref{fig:U_plot} -- \ref{fig:CC_plot} unequivocally indicate, the pragmatist approach is the optimal choice for normal queries in terms of \#Ax. Also wrt.\ \#Q, the pragmatist is the most reasonable expert type, although there are ontologies for which other approaches are better---but, if so, then just marginally. For instance, for ontology C, the maximalist strategy is the best choice when the number of queries should be minimized. The minimalist answering behavior, in contrast, was never the best strategy to minimize \#Q in our experiments. However, as argued in Sec.~\ref{sec:discussion_of_existing_approaches}, we believe that \#Ax is the more reasonable and realistic effort metric. In this view, the pragmatist answering style, where all query-axioms until the first negative one are classified and all others are left unclassified, is clearly the most efficient one.

So, the pragmatist approach appears to be the best trade-off between effort of query answering and achieved gain in terms of diagnoses discrimination. While this result is not self-evident at all,
a likely explanation for it is the following. When compared to the maximalist approach, the gain per axiom among the additional axioms classified after having found the first negative axiom is lower than the gain of the first axioms classified (cf.\ the ``law of diminishing returns''). In comparison with the minimalist strategy, it seems that positively classified query-axioms (before the first negative axiom is found), do bring a significant gain as compared to not classifying them. 

This matter of fact is quite well exposed in Figures~\ref{fig:U_plot} -- \ref{fig:CC_plot}, which show that 
the number of queries remains approximately the same for all axiom-based answering methods, whereas the number of inspected axioms is minimal for the pragmatist approach.
%

\noindent\emph{(Singleton queries:)} 
All four expert types coincide for singleton queries (cf.\ Sec.~\ref{sec:discussion_of_existing_approaches}). 
\vspace{5pt}

\noindent \emph{Ad \ref{Q3}:} \emph{(Query-based expert:)} If the effort metric \#Ax is considered, singleton queries are distinctly the interaction method of choice, as clearly evidenced by Figures~\ref{fig:U_plot} -- \ref{fig:CC_plot}. The cost overhead in terms of \#Ax when relying on normal instead of singleton queries amounts to up to over 200\,\% (e.g., CC ontology, RIO heuristic). 
Hence, although normal queries are optimized based on an analysis focusing on the query-based user, singleton queries are drastically more efficient in this scenario. This has two reasons. First, singleton queries are optimized for the query-based expert as well. Because they are---trivially---optimized for all discussed types of users, as all of them behave alike when asked singleton queries.
Second, classifying singleton queries brings more information per inspected query-axiom than classifying normal queries. Especially in case the given answer is negative, a singleton query pinpoints a faulty axiom whereas a normal query (including more than one axiom) just indicates that (any) one of its comprised axioms is faulty. 

On the other hand, when measuring the effort by \#Q, then there are cases where normal queries, and others where singleton queries are better. In concrete terms, singleton queries, on average, prevail over normal ones in 
all but two (i.e., K and D)
of the investigated ontologies. 
The cause of this lies in the fact that normal (non-singleton) queries provide more information than singleton ones given a positive query answer (multiple vs.\ a single axiom added to positive test cases), whereas the reverse is true for a negative answer (one \emph{undefined} axiom of multiple asserted wrong vs.\ one \emph{particular} axiom declared wrong). Obviously, the positive impact of singleton queries in the negation case compared ot normal queries outweighs the reduced gain in the affirmation case in the majority of examined scenarios.  

\emph{(Axiom-based expert:)} 
Studying Figures~\ref{fig:U_plot} -- \ref{fig:CC_plot} and comparing the best axiom-based answering strategy for normal queries, namely the pragmatist approach (see \ref{Q2} above), with singleton queries, we find that, in most cases, the singleton querying method is superior to normal querying as regards \#Ax. To illustrate this observation in more detail, Figure~\ref{fig:overhead_normal_vs_singleton} shows the incurred overhead in terms of the average \#Ax for all heuristics $h_1$ when using normal queries as compared to singleton queries. For instance, for the three heuristics analyzed in Figures~\ref{fig:U_plot} -- \ref{fig:CC_plot}, we find that singleton queries reduce the expert costs on average for 7 of 9 investigated ontologies when using ENT or SPL, and even in 8 of 9 cases for RIO. Averaged over all ontologies, we notice that the highest expected cost reduction by using singleton queries instead of normal ones is achieved by RIO (see rightmost area in Figure~\ref{fig:overhead_normal_vs_singleton}). 

\begin{figure}[ht]   
	\centering
	\footnotesize
	\includegraphics[width=\linewidth]{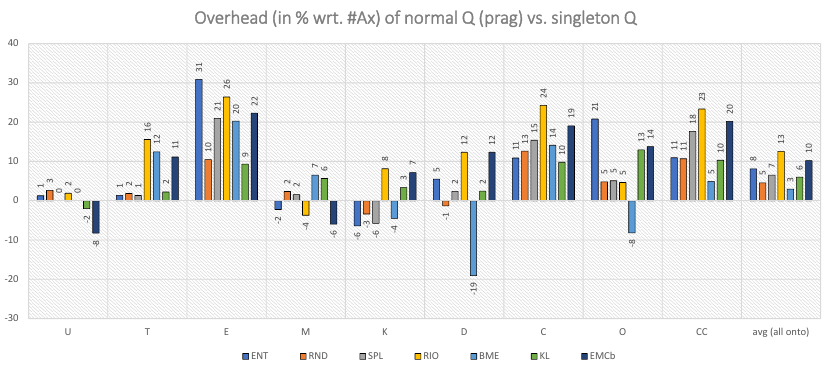}
	\caption{\small Bar chart showing the average overhead in \% regarding the number of classified axioms (\#Ax) per diagnostic session, grouped by ontology (x-axis) and heuristics (different colors). }
	\label{fig:overhead_normal_vs_singleton}
\end{figure}

However, when looking at the single fault localization sessions, there is a significant number of cases where normal queries answered by the pragmatist approach are the best choice wrt.\ \#Ax. This is well illustrated by the violin plots shown in Figures~\ref{fig:vioplot_ENT} -- \ref{fig:vioplot_EMCb}. 
%
%
%
While singleton queries are the equally good or better choice in the majority of cases (white dot at or above the red line) for \emph{all} ontologies when using the heuristics RND, ENT or KL, for eight of nine ontologies in case of RIO, BME or EMCb, and for seven of nine in case of SPL, we nevertheless realize that a significant area of almost all violin plots is below the red line. This area denotes the proportion of sessions where normal queries outperformed singleton ones. 
From this we discern that normal queries \emph{are} a reasonable way of interaction with an expert, but can match up to singleton queries only if the pragmatist answering behavior is given. Hence, existing systems relying on normal queries should advise their users to follow this approach to minimize their debugging time and effort. 

On the other hand, when the aim is to minimize \#Q, the picture looks at lot different. Here, all axiom-based strategies used in combination with normal queries outperform singleton queries---for all investigated ontologies (see Figures~\ref{fig:U_plot} -- \ref{fig:CC_plot}). This situation, however, is absolutely expected and its explanation is straightforward. Because, first, normal queries are computed with the aim to minimize exactly \#Q, while being selected from a query space that strictly subsumes the space of singleton queries.
Second, normal queries generally comprise multiple axioms, and all axiom-based users classify multiple of these axioms per query (averaged over positive and negative answers). Third, the metric \#Q abstracts from the effort in terms of classified axioms and counts just the number of asked queries. 
In the light of these aspects it is clear that fewer normal queries suffice to gather the same information as obtained using a higher number of singletons. This tells us that normal queries are the best choice when \#Q is the metric to be minimized.  
%
\vspace{5pt}

\noindent \emph{Ad \ref{Q4}:} Drawing our attention to the lines in Figures~\ref{fig:U_plot} -- \ref{fig:CC_plot}, we clearly recognize that singleton queries require significantly less computation time than normal queries (regardless of the answering strategy\footnote{It may seem unnecessary to differentiate between different answering strategies when considering the query computation time. However, each answering behavior involves different numbers and types of test cases that are added upon a query's answer, and these can, in theory, affect the computation time of prospective queries.}). In numerical terms, the savings in average computation time per query through the usage of singleton queries instead of normal ones amounts to between 80\,\% and 90\,\% over all ontologies. Please note, however, that absolute calculation times per query (stages~1 and 2, cf.\ Sec.~\ref{sec:new_approach:generation+optimization}) are very low (in the range of a few milliseconds) for both normal and singleton queries in all studied cases. Consequently, the time is definitely not a tie-breaker when deciding between both approaches. The justification for the computation speed-up when drawing on singleton queries as opposed to normal queries is the lower number of q-partitions that need to be explored in stage~1 of Algorithm~\ref{algo:singleton_query_selection} (cf.\ the ``smaller search space'' discussion in Sec.~\ref{sec:new_approach:properties}). This can be well read from Figures~\ref{fig:U_plot} -- \ref{fig:CC_plot}, where the red line (query computation time) changes proportionally to the black line (generated q-partitions). 
\vspace{5pt}

\noindent \emph{Ad \ref{Q5}:} As the analyses and argumentations for \ref{Q1} -- \ref{Q4} elucidate, singleton queries are by and large the best choice in case one would develop a debugging tool from scratch. The reasons for this conclusion in favor of singleton queries are---besides the pros enumerated in Sec.~\ref{sec:new_approach:properties}---their simplicity (interacting users need no advise whatsoever regarding the best answering strategy, etc.), their optimality and same performance achieved for all discussed expert types (all expert types coincide for singletons), their time-efficiency (faster computation), and their superior performance in the majority of cases over normal queries (fewer required expert interactions for fault localization).  

In case of already existing systems that draw on normal queries, experts should be advised to act according to the pragmatist answering strategy. In this case, an average performance comparable to singleton queries will be achieved.

\section{Research Limitations and Future Work}
\label{sec:research_limitations}
The primary aim of this paper is to assess the usefulness of the new singleton query type for interactive fault localization in ontologies. As our results reveal, singleton queries indeed provide a reasonable and efficient means for expert consultation and, altogether, outperform existing interaction techniques. Thus, this work on the one hand testifies that fault localization using singleton queries is a promising topic for further research, and on the other hand provides first results in this direction. 

However, this work also comes with limitations.
First, our evaluations are based on simulations of debugging sessions and objective measures such as the number of required queries or classified axioms, or the computation times. Beside this objective assessment, of course, it is important to validate the subjective usability and acceptance of the approach, for instance in terms of a user study. This is part of our future work. However, we are nevertheless confident that users who are familiar using normal queries would likewise accept and adopt singleton queries. The first argument in this regard is that normal queries might be singletons as well, simply because they can contain \emph{one} or more axioms. Second, there is no retraining or relearning whatsoever required to switch from the usage of normal queries to singletons, regardless of whether the expert is a query- or axiom-based type, since both querying approaches ask the user the same question, whether the set (or conjunction) of query-axioms is an entailment of the intended ontology. 
In fact, singleton queries even provide less space for misunderstandings and are easier explained to the user than normal ones as the implication of the \emph{set of} axioms does not need to be clarified.
Due to these points, we believe that the main results regarding the effectiveness of the query-based approach we obtained in our past user study \cite{rodler2019userstudy} conducted for normal queries can be transferred to singletons as well. 

A second limitation is the restriction to so-called explicit queries \cite{DBLP:journals/corr/Rodler16a}---those that are constituted by axioms from the ontology at hand---in our theoretical and empirical analyses. 
The reason we did so is because we were able to devise an algorithm for the computation and optimization of explicit queries, by drawing on and extending the 
theory elaborated by \cite{DBLP:journals/corr/Rodler16a}. 
The finding of an \emph{efficient} algorithm that soundly generates implicit singleton queries, in contrast, is an open issue and on our future work agenda. As discussed in Sec.~\ref{sec:discussion_of_existing_approaches}, this difficulty also explains why current approaches restrict themselves to normal (and not singleton) queries. 
Implicit queries are interesting particularly from the point of view of query complexity, i.e., how well an expert might understand (the axioms in) the query. A (syntax-based) model for estimating this complexity is suggested and evaluated in \cite{rodler2019userstudy}. According to it, e.g., axioms like $\mathit{A \,SubClassOf\, B}$ are easier to comprehend for users if $A,B$ are atomic classes rather than complex class expressions involving, e.g., negation or property restrictions. Whereas the syntactic (or structural) complexity of explicit queries depends on the complexity of (the axioms in) the ontology, the shape of implicit queries can be controlled subject to the options offered by the used Description Logic reasoner \cite{DLHandbook}. Reasoners such as Pellet \cite{sirin2007pellet} or HermiT \cite{Shearer2008}, for example, can be configured to restrict the computation of entailments to only specific axiom types, e.g., simple class subsumptions as mentioned above or basic class assertion axioms. In spite of this advantage over explicit queries, the use of \emph{only} implicit queries leads to the loss of 
the guarantee \cite[Prop.~7.5]{Rodler2015phd} that there is always a query to discriminate between two competing diagnoses. This underscores the importance of explicit queries, as discussed in this work.
%
%
%

As a third limitation, it should be noted that the analyzed expert types, as discussed in Sec.~\ref{sec:discussion_of_existing_approaches}, provide by no means a complete characterization of all possible cases that could arise. While the discussion in this work bases on the assumption that
an expert will provide for each query at the minimum as much information as is necessary to classify the entire query as a positive or negative test case (cf.\ the $\oracle$ function in Sec.~\ref{sec:basics}), there are (at least) two further query answering scenarios that are worthwhile considering. First, there is the case where the expert classifies a proper subset (or even none) of the axioms of a normal query positively while not labeling any axiom negatively, e.g., due to laziness or lack of knowledge.
In such a scenario, the expert does not ``implement'' the $\oracle$ function, as their answer leaves the classification of the query open---it could be negative if some of the remaining non-classified axioms is actually a non-entailment, or positive of all of them are entailments. Second, there is the case where an expert might misclassify axioms when answering queries. Such ``oracle errors'' were observed quite commonly in the studies conducted by \cite{rodler2019userstudy}. Investigating these scenarios for normal and singleton queries as well as the conception of strategies how to handle these cases is another research avenue we will prospectively pursue.  

In the light of these aspects, this work is just a first step towards understanding the impact of different interaction modes with users in the ontology fault localization domain. With the suggested singleton queries as an interface between expert and debugging system, however, it also gives a strategy that makes the overall fault finding process more efficient while not rendering the task more complicated. 

\section{Conclusions}
\label{sec:conclusion}
We observe that existing approaches to query-based fault localization in ontologies interact with an expert by means of batch questions. That is, an expert is asked to classify \emph{a set of} axioms as either a positive test case (the conjunction of axioms in the set is an entailment of the intended ontology) or as a negative one (some axiom is not an intended entailment). 
We point out that, on the one hand, there is a multitude of variants how an expert might answer such batch queries. In particular, we differentiate between four different expert types with regard to their query responses. On the other hand, current approaches
ground the computation, selection and optimization of batch queries on the assumption of one particular of these answering behaviors. Since violations of this assumption turn optimizations into approximations and might lead to unexpected results and worse efficiency of the interactive fault localization process, we suggest as a remedy to use singleton queries, i.e., queries including exactly one axiom, to consult an expert. 
We elaborate a theory of computation and (heuristics-based) optimization of singleton queries and provide complexity results for the suggested poly-time and poly-space algorithm.

Besides several apparent advantages of singleton queries in comparison to normal ones---such as a smaller search space, the facilitation of a precise (non-approximate) a-priori query assessment, or a more informative expert feedback---we conduct comprehensive empirical evaluations to gauge the usefulness of the new querying approach with regard to the expert waiting time between two queries and the effort necessary to locate the faulty axioms in the ontology. The main conclusions drawn from this study are:
\begin{enumerate}[noitemsep]
	\item Singleton queries are the overall best means of user consultation. The required expert interactions in terms of classified axioms are lower than for batch queries in the majority of diagnostic sessions, for almost all examined scenarios. Moreover, the time required for query computation and optimization is reduced by 80 to 90\,\% when using singletons. In absolute terms, it takes the proposed algorithm just a few milliseconds to obtain the heuristically-optimal query in the entire query search space. Furthermore, singleton queries are simpler and equally well suited for all different discussed query answering behaviors.
	\item For batch queries, we find that there is a significant difference regarding the required expert interactions for fault localization for the various discussed query answering styles, with the best strategy being the chronological evaluation of axioms in the query until the first negatively classified one (if any) is found. In particular, this leads to less expert effort than classifying (i)~all axioms per query or (ii)~just a minimal subset of the query-axioms. When experts are properly advised to pursue the right answering strategy, then the costs of batch queries are comparable to singleton ones. This shows that both batch and singleton queries are, in general, reasonable approaches.  
\end{enumerate} 
Finally, it is worthwhile noting that this approach is generally applicable for any knowledge representation language for which the entailment relation is monotonic (cf.\ \cite{Rodler2015phd}), e.g., Horn Logic, Propositional Logic, diverse constraint languages, as well as for other model-based diagnosis applications, as shown by \cite{rodler17dx_reducing}. 
%
%
%
%
%
%
%

 \bibliographystyle{splncs04}
 \bibliography{literature}

\begin{thebibliography}{10}
\providecommand{\url}[1]{\texttt{#1}}
\providecommand{\urlprefix}{URL }
\providecommand{\doi}[1]{https://doi.org/#1}

\bibitem{DLHandbook}
Baader, F., Calvanese, D., McGuinness, D., Nardi, D., Patel-Schneider, P.
  (eds.): {The Description Logic Handbook: Theory, Implementation, and
  Applications}. Cambridge University Press, 1st edn. (2003)

\bibitem{beck2003test}
Beck, K.: Test-driven development: by example. Addison-Wesley Professional
  (2003)

\bibitem{ceusters2005terminological}
Ceusters, W., Smith, B., Goldberg, L.: A terminological and ontological
  analysis of the nci thesaurus. Methods of information in medicine
  \textbf{44}(4), ~498 (2005)

\bibitem{felfernig2004consistency}
Felfernig, A., Friedrich, G., Jannach, D., Stumptner, M.: Consistency-based
  diagnosis of configuration knowledge bases. Artificial Intelligence
  \textbf{152}(2),  213--234 (2004)

\bibitem{golbeck2003NCIT}
Golbeck, J., Fragoso, G., Hartel, F., Hendler, J., Oberthaler, J., Parsia, B.:
  The national cancer institute's thesaurus and ontology. Journal of Web
  Semantics First Look 1\_1\_4  (2003)

\bibitem{grau2008owl}
Grau, B.C., Horrocks, I., Motik, B., Parsia, B., Patel-Schneider, P., Sattler,
  U.: Owl 2: The next step for owl. Web Semantics: Science, Services and Agents
  on the World Wide Web  \textbf{6}(4),  309--322 (2008)

\bibitem{horridge2011cognitive}
Horridge, M., Bail, S., Parsia, B., Sattler, U.: The cognitive complexity of
  owl justifications. In: International Semantic Web Conference. pp. 241--256.
  Springer (2011)

\bibitem{hyafil1976}
Hyafil, L., Rivest, R.L.: Constructing optimal binary decision trees is
  np-complete. Information processing letters  \textbf{5}(1),  15--17 (1976)

\bibitem{jannach2016parallel}
Jannach, D., Schmitz, T., Shchekotykhin, K.: Parallel model-based diagnosis on
  multi-core computers. Journal of Artificial Intelligence Research
  \textbf{55},  835--887 (2016)

\bibitem{Kalyanpur2006a}
Kalyanpur, A.: {Debugging and Repair of OWL Ontologies}. Ph.D. thesis,
  University of Maryland, College Park (2006)

\bibitem{dekleer1992onesteplookahead}
de~Kleer, J., Raiman, O., Shirley, M.: One step lookahead is pretty good. In:
  Readings in model-based diagnosis. pp. 138--142. Morgan Kaufmann Publishers
  Inc. (1992)

\bibitem{dekleer1987}
de~Kleer, J., Williams, B.C.: {Diagnosing multiple faults}. Artificial
  Intelligence  \textbf{32}(1),  97--130 (Apr 1987)

\bibitem{meilicke2011}
Meilicke, C.: Alignment incoherence in ontology matching. Ph.D. thesis,
  Universit{\"a}t Mannheim (2011)

\bibitem{Nikitina11}
Nikitina, N., Rudolph, S., Glimm, B.: Interactive ontology revision. J. Web
  Sem.  \textbf{12}(0) (2012),
  \url{http://www.websemanticsjournal.org/index.php/ps/article/view/233}

\bibitem{noy2003protege}
Noy, N.F., Crub{\'e}zy, M., Fergerson, R.W., Knublauch, H., Tu, S.W., Vendetti,
  J., Musen, M.A.: Prot{\'e}g{\'e}-2000: an open-source ontology-development
  and knowledge-acquisition environment. In: AMIA 2003 Open Source Expo. pp.
  953--953 (2003)

\bibitem{quinlan1986induction}
Quinlan, J.R.: Induction of decision trees. Machine learning  \textbf{1}(1),
  81--106 (1986)

\bibitem{rector2011getting}
Rector, A.L., Brandt, S., Schneider, T.: Getting the foot out of the pelvis:
  modeling problems affecting use of snomed ct hierarchies in practical
  applications. Journal of the American Medical Informatics Association
  \textbf{18}(4),  432--440 (2011)

\bibitem{Reiter87}
Reiter, R.: {A Theory of Diagnosis from First Principles}. Artificial
  Intelligence  \textbf{32}(1),  57--95 (1987)

\bibitem{Rodler2015phd}
Rodler, P.: {Interactive Debugging of Knowledge Bases}. Ph.D. thesis,
  Alpen-Adria Universit\"at Klagenfurt (2015)

\bibitem{DBLP:journals/corr/Rodler16a}
Rodler, P.: Towards better response times and higher-quality queries in
  interactive knowledge base debugging. Tech. rep., Alpen-Adria Universit\"at
  Klagenfurt (2016), http://arxiv.org/pdf/1609.02584v2.pdf

\bibitem{rodler17dx_activelearning}
Rodler, P.: On active learning strategies for sequential diagnosis. In: 28th
  International Workshop on Principles of Diagnosis (DX'17). vol.~4, pp.
  264--283 (2018)

\bibitem{rodler2022onestep}
Rodler, P.: One step at a time: An efficient approach to query-based ontology
  debugging. Knowl. Based Syst. \textbf{108987} (2022)

\bibitem{rodler2018socs}
Rodler, P., Herold, M.: {StaticHS}: {A} variant of {Reiter}'s hitting set tree
  for efficient sequential diagnosis. In: Proceedings of the Eleventh
  International Symposium on Combinatorial Search, {SOCS} 2018, Stockholm,
  Sweden - 14-15 July 2018. pp. 72--80 (2018)

\bibitem{rodler2019userstudy}
Rodler, P., Jannach, D., Schekotihin, K., Fleiss, P.: Are query-based ontology
  debuggers really helping knowledge engineers? Knowl. Based Syst.
  \textbf{179},  92--107 (2019)

\bibitem{rodler17dx_reducing}
Rodler, P., Schekotihin, K.: Reducing model-based diagnosis to knowledge base
  debugging. In: 28th International Workshop on Principles of Diagnosis
  (DX'17). vol.~4, pp. 284--296 (2018)

\bibitem{rodler2018ruleML}
Rodler, P., Schmid, W.: On the impact and proper use of heuristics in
  test-driven ontology debugging. In: Rules and Reasoning - Second
  International Joint Conference, RuleML+RR 2018, Luxembourg, September 18-21,
  2018, Proceedings. pp. 164--184 (2018)

\bibitem{DBLP:journals/corr/rodler17jair}
Rodler, P., Schmid, W., Schekotihin, K.: A generally applicable, highly
  scalable measurement computation and optimization approach to sequential
  model-based diagnosis. CoRR abs/1711.05508  (2017),
  \url{http://arxiv.org/abs/1711.05508}

\bibitem{rodler-dx17}
Rodler, P., Schmid, W., Schekotihin, K.: Inexpensive cost-optimized measurement
  proposal for sequential model-based diagnosis. In: 28th International
  Workshop on Principles of Diagnosis (DX'17). vol.~4, pp. 200--218 (2018)

\bibitem{Rodler2013}
Rodler, P., Shchekotykhin, K., Fleiss, P., Friedrich, G.: {RIO: Minimizing User
  Interaction in Ontology Debugging}. In: Web Reasoning and Rule Systems, pp.
  153--167 (2013)

\bibitem{russellnorvig2016}
Russell, S.J., Norvig, P.: Artificial intelligence: a modern approach.
  Malaysia; Pearson Education Limited, (2016)

\bibitem{schekotihin2018ontodebug}
Schekotihin, K., Rodler, P., Schmid, W.: Ontodebug: Interactive ontology
  debugging plug-in for prot{\'e}g{\'e}. In: International Symposium on
  Foundations of Information and Knowledge Systems. pp. 340--359. Springer
  (2018)

\bibitem{DBLP:conf/icbo/SchekotihinRSHT18a}
Schekotihin, K., Rodler, P., Schmid, W., Horridge, M., Tudorache, T.: A
  prot{\'{e}}g{\'{e}} plug-in for test-driven ontology development. In:
  Proceedings of the 9th International Conference on Biological Ontology
  {(ICBO} 2018), Corvallis, Oregon, USA, August 7-10, 2018. (2018),
  \url{http://ceur-ws.org/Vol-2285/ICBO\_2018\_paper\_9.pdf}

\bibitem{schulz2010pitfalls}
Schulz, S., Schober, D., Tudose, I., Stenzhorn, H.: The pitfalls of thesaurus
  ontologization--the case of the nci thesaurus. In: AMIA Annual Symposium
  Proceedings. vol.~2010, p.~727. American Medical Informatics Association
  (2010)

\bibitem{Shchekotykhin2012}
Shchekotykhin, K., Friedrich, G., Fleiss, P., Rodler, P.: {Interactive Ontology
  Debugging: Two Query Strategies for Efficient Fault Localization}. Web
  Semantics: Science, Services and Agents on the World Wide Web
  \textbf{12-13},  88--103 (2012)

\bibitem{shchekotykhin2014sequential}
Shchekotykhin, K.M., Friedrich, G., Rodler, P., Fleiss, P.: Sequential
  diagnosis of high cardinality faults in knowledge-bases by direct diagnosis
  generation. In: ECAI. vol.~14, pp. 813--818 (2014)

\bibitem{shchekotykhin2015mergexplain}
Shchekotykhin, K.M., Jannach, D., Schmitz, T.: Mergexplain: Fast computation of
  multiple conflicts for diagnosis. In: IJCAI. vol.~15, pp. 3221--3228 (2015)

\bibitem{Shearer2008}
Shearer, R., Motik, B., Horrocks, I.: Hermit: {A} highly-efficient {OWL}
  reasoner. In: {OWLED}. {CEUR} Workshop Proceedings, vol.~432 (2008)

\bibitem{siddiqi2007hierarchical}
Siddiqi, S., Huang, J.: Hierarchical diagnosis of multiple faults. In: IJCAI.
  pp. 581--586 (2007)

\bibitem{sirin2007pellet}
Sirin, E., Parsia, B., Grau, B.C., Kalyanpur, A., Katz, Y.: {Pellet: A
  practical OWL-DL reasoner}. Journal of Web Semantics  \textbf{5}(2),  51--53
  (2007)

\bibitem{smith2007obo}
Smith, B., Ashburner, M., Rosse, C., Bard, J., Bug, W., Ceusters, W., Goldberg,
  L.J., Eilbeck, K., Ireland, A., Mungall, C.J., et~al.: The obo foundry:
  coordinated evolution of ontologies to support biomedical data integration.
  Nature biotechnology  \textbf{25}(11), ~1251 (2007)

\bibitem{stuckenschmidt2008debugging}
Stuckenschmidt, H.: Debugging owl ontologies-a reality check. In: EON. vol.~359
  (2008)

\bibitem{tudorache2008supporting}
Tudorache, T., Noy, N.F., Tu, S., Musen, M.A.: Supporting collaborative
  ontology development in prot{\'e}g{\'e}. In: International Semantic Web
  Conference. pp. 17--32. Springer (2008)

\end{thebibliography}
 
 \begin{sidewaysfigure}
 	\begin{minipage}{0.48\linewidth}
 		
 		\centering
 		\includegraphics[width=9cm]{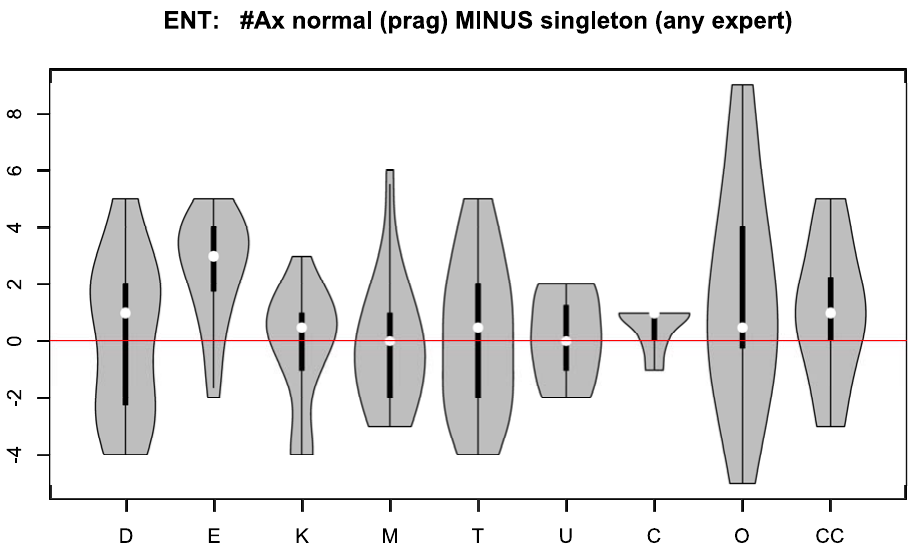}
 		\caption{\small ENT heuristic.}
 		\label{fig:vioplot_ENT}
 		
 		\vspace{20pt}		
 		
 		\centering
 		\includegraphics[width=9cm]{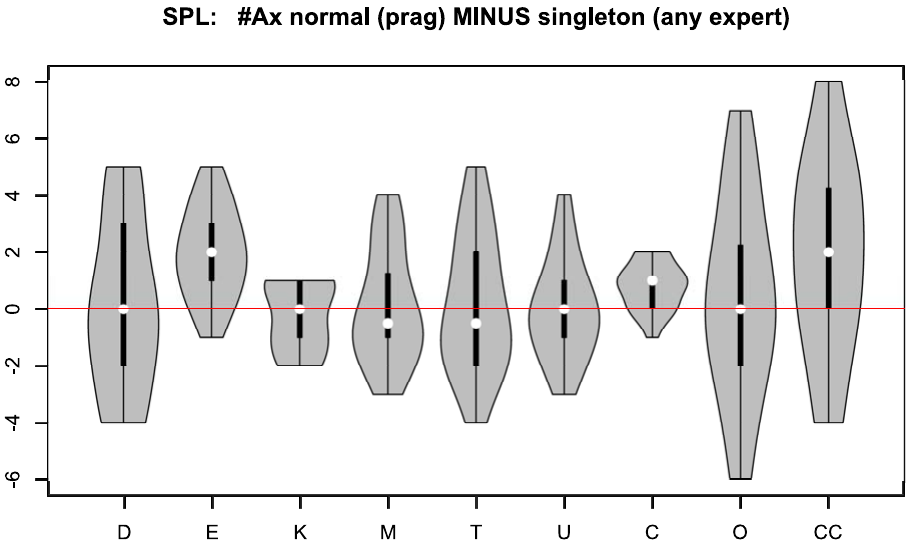}
 		\caption{\small SPL heuristic.}
 		\label{fig:vioplot_SPL}
 		
 	\end{minipage}
 	\hfill
 	\begin{minipage}{0.48\linewidth}
 		
 		\centering
 		\includegraphics[width=9.2cm]{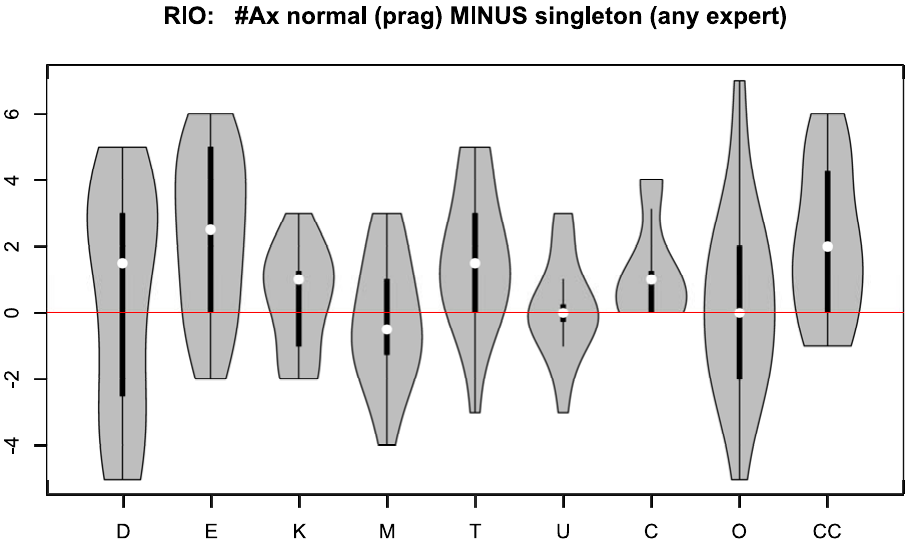}
 		\caption{\small RIO heuristic.}
 		\label{fig:vioplot_RIO}
 		
 		\vspace{20pt}	
 		
 		\centering
 		\includegraphics[width=9.2cm]{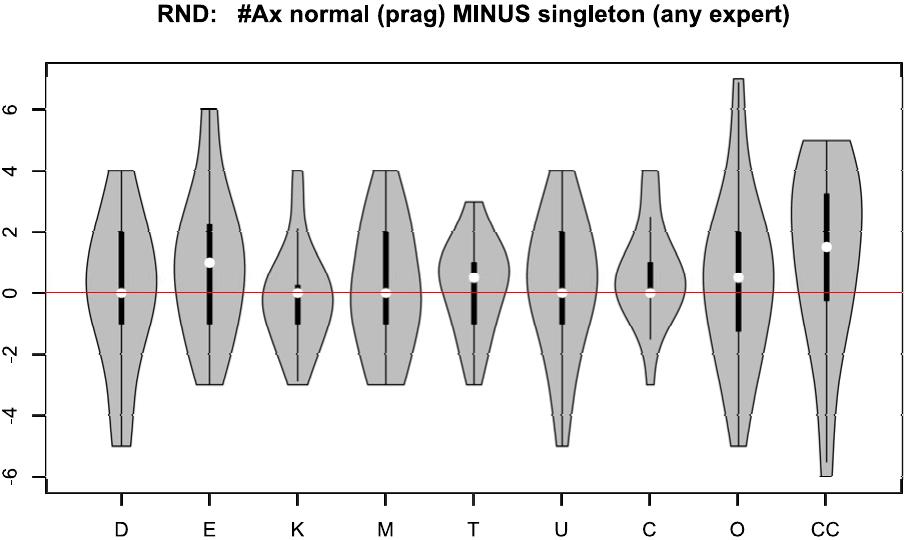}
 		\caption{\small RND heuristic.}
 		\label{fig:vioplot_RND}
 		
 	\end{minipage}
 \end{sidewaysfigure}

 \begin{sidewaysfigure}
	\begin{minipage}[t][][t]{0.48\linewidth}
		
		\centering
		\includegraphics[width=9cm]{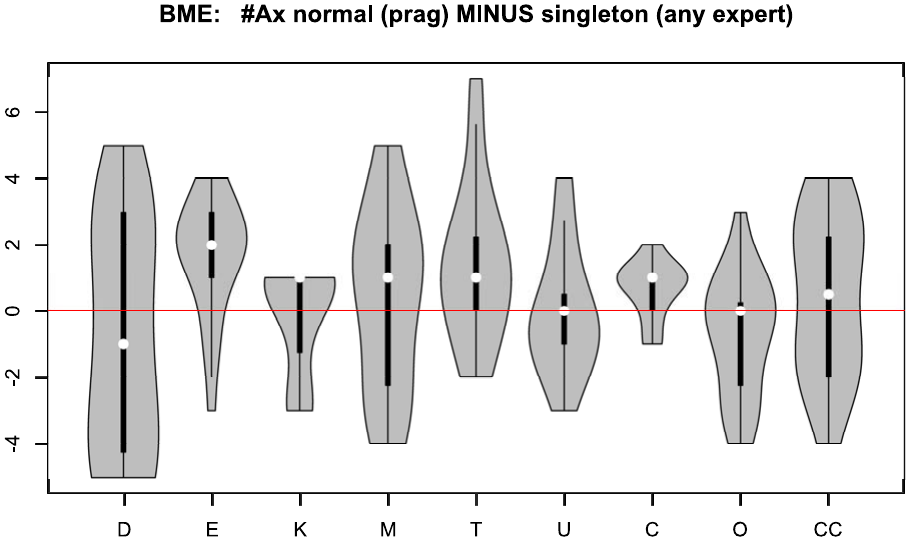}
		\caption{\small BME heuristic.}
		\label{fig:vioplot_BME}
		
		\vspace{20pt}		
		
		\centering
		\includegraphics[width=9cm]{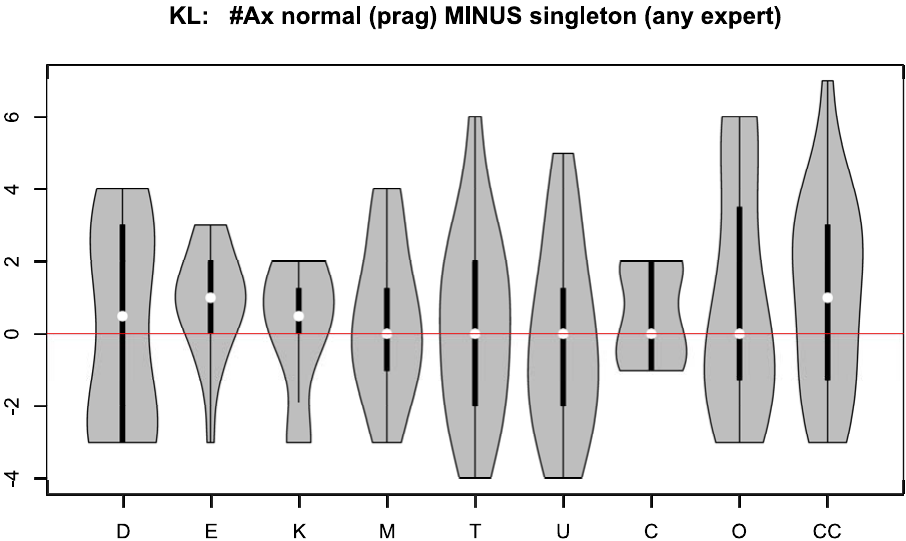}
		\caption{\small KL heuristic.}
		\label{fig:vioplot_KL}

	\end{minipage}
	\hfill
	\begin{minipage}[t][][t]{0.48\linewidth}
		\centering
		\includegraphics[width=9cm]{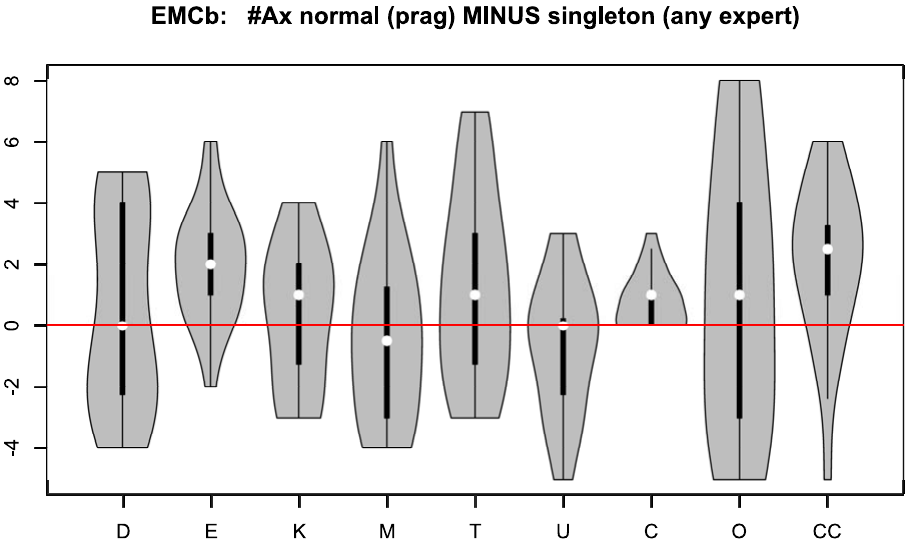}
		\caption{\small EMCb heuristic.}
		\label{fig:vioplot_EMCb}
		
		\vspace{40pt}
		
		\centering
		\includegraphics[width=9.2cm]{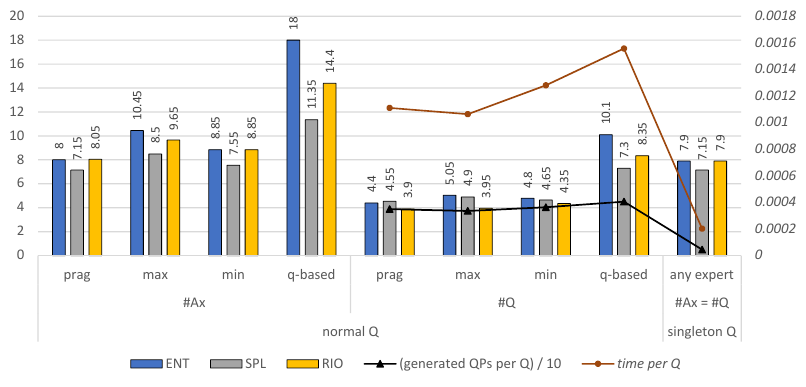}
		\caption{\small U ontology overview}
		\label{fig:U_plot}

	\end{minipage}
\end{sidewaysfigure}

\begin{sidewaysfigure}
	\begin{minipage}[t][][t]{0.50\textwidth}
		
		\centering
		\includegraphics[width=9.5cm]{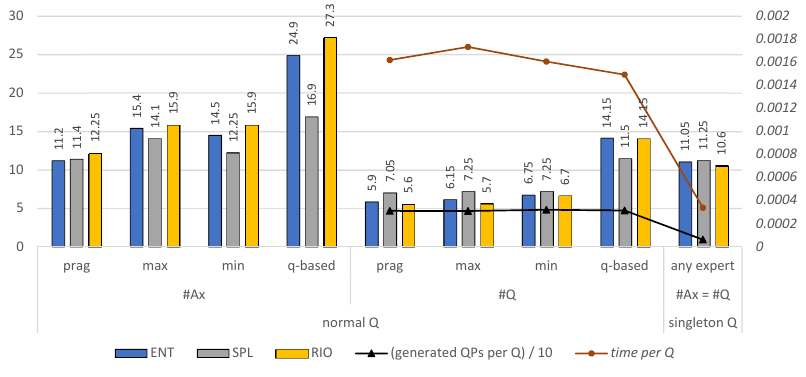}
			\vspace{-10pt}
		\caption{\small T ontology overview.}
		\label{fig:T_plot}	
		
		\vspace{20pt}
		\centering
		\includegraphics[width=9.5cm]{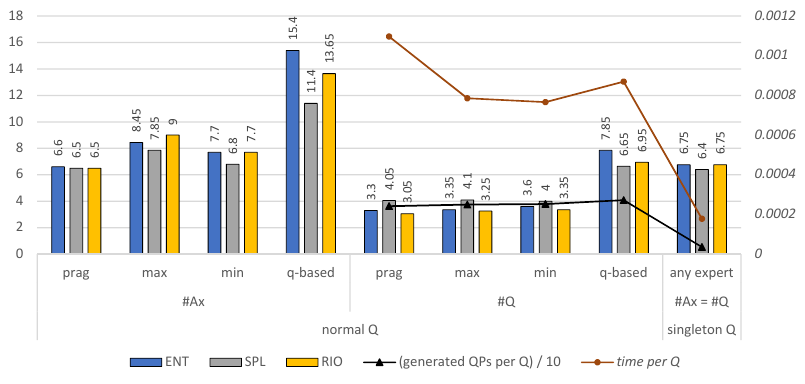}
			\vspace{-10pt}
		\caption{\small M ontology overview.}
		\label{fig:M_plot}

	\end{minipage}
	\hfill
	\begin{minipage}[t][][t]{0.50\textwidth}

		\centering
		\includegraphics[width=9.5cm]{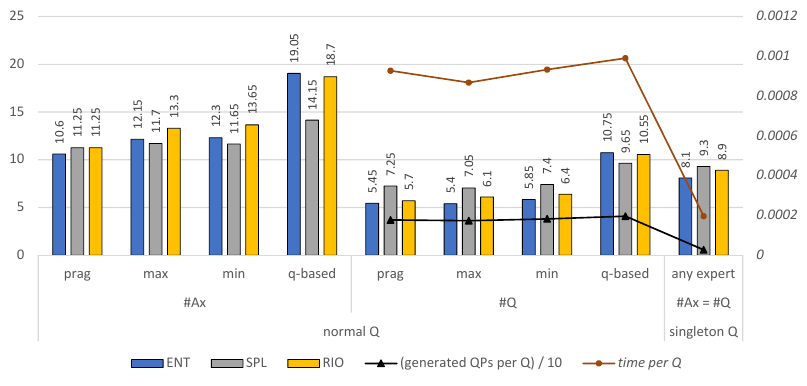}
		\vspace{-10pt}
		\caption{\small E ontology overview.}
		\label{fig:E_plot}
		
		\vspace{20pt}
		\centering
		\includegraphics[width=9.5cm]{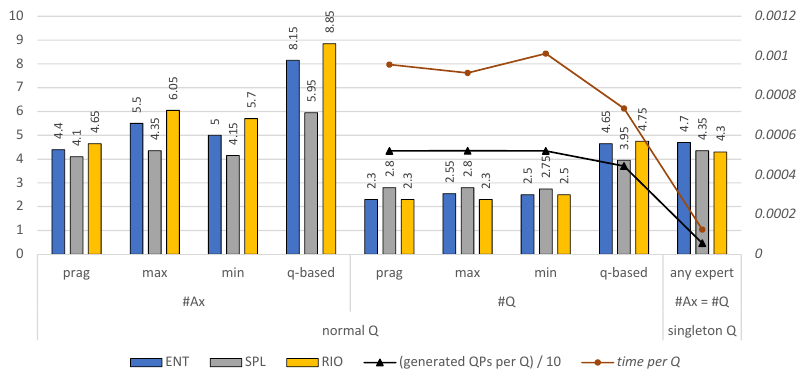}
			\vspace{-10pt}
		\caption{\small K ontology overview.}
		\label{fig:K_plot}

	\end{minipage}
\end{sidewaysfigure}

\begin{sidewaysfigure}
		\begin{minipage}[t][][t]{0.50\textwidth}
		
		\centering
		\includegraphics[width=9.5cm]{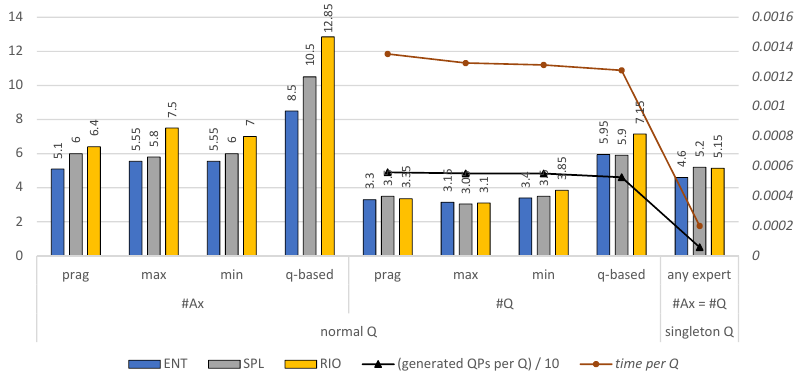}
		\vspace{-10pt}
		\caption{\small C ontology overview.}
		\label{fig:C_plot}
		
		\vspace{20pt}
		\centering
		\includegraphics[width=9.5cm]{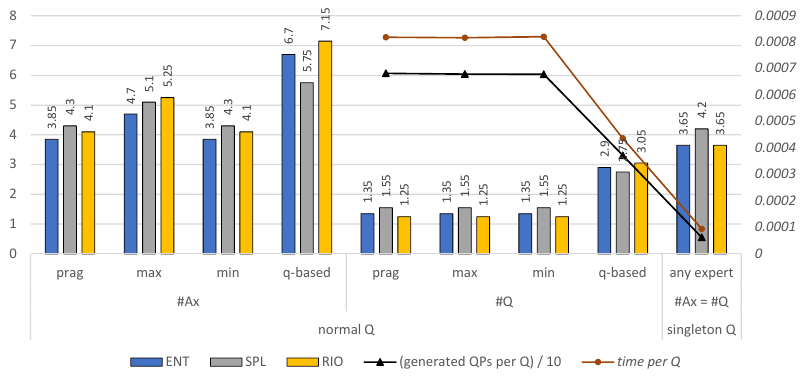}
		\vspace{-10pt}
		\caption{\small D ontology overview.}
		\label{fig:D_plot}

	\end{minipage}
	\hfill
	\begin{minipage}[t][][t]{0.50\textwidth}
	
	\centering
	\includegraphics[width=9.5cm]{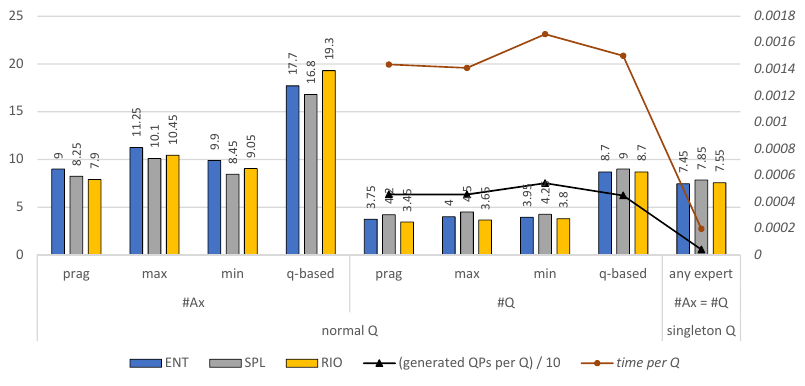}
	\vspace{-10pt}
	\caption{\small O ontology overview.}
	\label{fig:O_plot}
	
	\vspace{20pt}
	\centering
	\includegraphics[width=9.5cm]{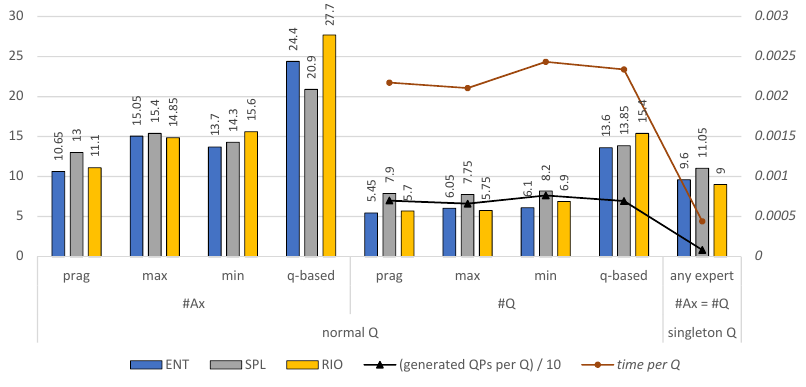}
	\vspace{-10pt}
	\caption{\small CC ontology overview.}
	\label{fig:CC_plot}
		
\end{minipage}
%
%
%
%

\end{sidewaysfigure}

\end{document}